\documentclass[11pt]{article}
\usepackage[paper=letterpaper,margin=1in]{geometry}

\usepackage[utf8]{inputenc} %
\usepackage[T1]{fontenc}    %
\usepackage{hyperref}       %
\usepackage{url}            %
\usepackage{booktabs}       %
\usepackage{amsfonts}       %
\usepackage{nicefrac}       %
\usepackage{microtype}      %

\usepackage{enumitem}
\usepackage{amsmath,amssymb,dsfont,amsthm,dsfont,hyperref}
\usepackage{xfrac}
\usepackage{xcolor}
\usepackage[title,toc,titletoc,page]{appendix}
\usepackage[linesnumbered,ruled]{algorithm2e}
\usepackage[numbers]{natbib}
\usepackage[capitalize]{cleveref}
\usepackage{subcaption}
\usepackage{forloop}

\usepackage{float}
\usepackage{wrapfig}
\restylefloat{figure}

\usepackage{microtype}
\usepackage{graphicx}

\usepackage{thmtools} 
\usepackage{thm-restate}

\newtheorem{theorem}{Theorem}
\newtheorem{lemma}[theorem]{Lemma}
\newtheorem{cor}[theorem]{Corollary}

\newtheorem{prop}[theorem]{Proposition}

\theoremstyle{definition}

\newcommand{\eps}{\varepsilon}
\newcommand{\epsp}{\varepsilon'}
\newcommand{\E}{\mathbb{E}}
\newcommand{\ind}{\mathds{1}}
\DeclareMathOperator{\erf}{\textnormal{erf}}
\DeclareMathOperator{\sign}{sign}
\DeclareMathOperator{\diag}{diag}

\DeclareMathOperator{\unif}{Unif}
\DeclareMathOperator{\pois}{Poisson}

\DeclareMathOperator*{\argmax}{arg\,max}
\DeclareMathOperator*{\argmin}{arg\,min}

\newcommand{\defcal}[1]{\expandafter\newcommand\csname 
	c#1\endcsname{{\mathcal{#1}}}}
\newcommand{\defbb}[1]{\expandafter\newcommand\csname 
	b#1\endcsname{{\mathbb{#1}}}}
\newcommand{\defbf}[1]{\expandafter\newcommand\csname 
	bf#1\endcsname{{\mathbf{#1}}}}
\newcounter{calBbCounter}
\forLoop{1}{26}{calBbCounter}{
	\edef\letter{\Alph{calBbCounter}}
	\expandafter\defcal\letter
	\expandafter\defbb\letter
	\expandafter\defbf\letter
}
\forLoop{1}{26}{calBbCounter}{
	\edef\letter{\alph{calBbCounter}}
	\expandafter\defbf\letter
}

\newcommand{\wrob}{w^\textnormal{rob}}

\newcommand{\frob}{f^\textnormal{rob}}
\newcommand{\dgau}{\cD_{\cN}}

\title{The Curious Case of Adversarially Robust Models:\\ More Data Can Help,  Double Descend, or  Hurt Generalization}
\author{Yifei Min\thanks{Department of Statistics and Data Science, Yale University.  E-mail: \texttt{yifei.min@yale.edu}.}
\and 
Lin Chen\thanks{Yale Institute for Network Science, Yale University.  E-mail: \texttt{lin.chen@yale.edu}.} 
\and
Amin Karbasi\thanks{Yale Institute for Network Science, Yale University.  E-mail: \texttt{amin.karbasi@yale.edu}.}
}
\date{}

\begin{document}

\maketitle

\begin{abstract}

    Adversarial training has shown its ability in producing models that are robust to perturbations on the input data, but usually at the expense of decrease in the standard accuracy. To mitigate this issue, it is commonly believed that more training data will eventually help such adversarially robust models generalize better on the benign/unperturbed test data.
    In this paper, however, we challenge this conventional belief and show that more training data can hurt the generalization of adversarially robust models in the classification problems. We first investigate the Gaussian mixture classification with a linear loss and identify three regimes based on the strength of the adversary. In the weak adversary regime, more data improves the generalization of adversarially robust models. In the medium adversary regime, with more training data, the generalization loss exhibits a double descent curve, which implies the existence of an intermediate stage where more training data hurts the generalization. In the strong adversary regime, more data almost immediately causes the generalization error to increase. Then we move to the analysis of a two-dimensional classification problem with a 0-1 loss. We prove that more data always hurts the generalization performance of adversarially trained models with large perturbations. To complement our theoretical results, we conduct empirical studies on Gaussian mixture classification, support vector machines (SVMs), and linear regression.

\end{abstract} %
\section{Introduction}\label{sec:intro}
In recent years, modern machine learning methods have exhibited their superiority over traditional models in an abundance of machine learning tasks, e.g., image classification~\citep{krizhevsky2012imagenet}, speech recognition and language translation~\citep{graves2013speech, bahdanau2015neural}, medical diagnosis~\citep{lakhani2017deep,xiao2019denxfpn}, text recognition and information extraction~\citep{long2020new,mei2018halo,wang2012end}, online fraud detection~\citep{pumsirirat2018credit}, and self-driving cars~\citep{ramos2017detecting}, among others. However, they can also be extremely vulnerable to adversarial,  human-imperceptible data modifications
~\citep{szegedy2013intriguing,carlini2018audio,kos2018adversarial}. This vulnerability is even more concerning and dangerous when machine learning methods are used in scenarios directly connected to human safety such as medical diagnosis (misinterpreting medical images) or self-driving cars (misreading traffic signs). %
To circumvent these issues, practitioners introduce adversarial training in order to produce adversarially robust models~\citep{huang2015learning,shaham2018understanding,madry2017towards,zhang2019you,gao2019convergence,song2018improving} that can still make consistently correct predictions, even when faced with perturbed data.

There is a large body of work dedicated to adversarially robust models~\citep{zhang2019interpreting,santurkar2019image,zhang2019theoretically,bhagoji2019lower,diochnos2019lower,wei2019improved,zhai2019adversarially}. In particular, it has been shown that there exists a trade-off between the generalization of a model (i.e. the standard accuracy) and its robustness to adversarial perturbation \citep{tsipras2018robustness}. %
Along a similar vein, \citet{schmidt2018adversarially} showed that adversarially robust models need more training data compared to their standard counterparts in order to achieve the same generalization performance. In this paper, we want to further investigate these ideas and explore whether simply adding more data is enough for adversarially robust models to catch up to the generalization ability of their standard counterparts.

Previous works have studied the generalization of adversarially robust models from a variety of perspectives. For instance, \citet{yin2019rademacher} and \citet{khim2018adversarial} gave bounds on the generalization error of adversarially robust models via Rademacher complexity. More recently, \citet{chen2020more} studied the influence of a larger training set upon the gap between the generalization performance of an adversarially robust model and a standard model. They proved that more training data could result in expansion of the gap and denied the belief that more training data always helps adversarially robust models reach a similar generalization performance to the standard model. Building on these works, our goal is to move past bounds and gaps, and directly characterize how the size of training set affects the accuracy of adversarially robust models on unperturbed test data.

\subsection{Our Contributions}

A conventional wisdom in machine learning is that a larger training set will result in better generalization on the test data. We provably establish a surprising, and to some extent even paradoxical, result that more training data can hurt the generalization of adversarially robust models.
We first consider a linear classification problem with a linear loss function and identify three regimes of different 
adversary strengths, i.e., the weak, medium, and strong adversary regimes. 

\begin{itemize}
    \item In the \textbf{strong adversary regime}, the generalization of adversarially 
    robust models deteriorates with more training data, except for a possible 
    short initial stage where the generalization is improved with more data. 
    \item The \textbf{medium adversary regime} is probably the most interesting one among the three regimes. In this regime, the evolution of the generalization performance of adversarially robust models could be a double descent curve.  In particular, at the initial stage, the generalization loss on the test data is reduced with more training data. At the intermediate stage, however, the generalization loss increases as there is more training data (more data hurts the generalization of adversarial robust models). At the final stage, more training data improves the generalization performance. 
    \item In the \textbf{weak adversary regime}, the generalization is consistently 
    improved with more training data.
\end{itemize}

We then move to the analysis of the 0-1 loss and investigate a two-dimensional classification problem where the candidate decision boundary is given by a piecewise constant function. 
Similar weak and strong adversary regimes are observed under this setting. In particular, in the strong adversary regime, more data always hurts the generalization of adversarially robust models. 

We complement the above theroetical results with empirical studies on important machine learning models, including support vector machines (SVMs), linear regression, and Gaussian mixture classification with 0-1 loss. We observe a similar phenomenon that more data hurts generalization in adversarial training. These empirical results suggest that the observed phenomenon may be ubiquitous across different models and loss functions and that we need to reflect on the true role that the size of the training set plays in adversarial training.

\section{Related Work}
In this section, we briefly discuss some additional papers on the generalization of adversarially robust models and the double descent phenomenon, which are most relevant to our work.

\citet{schmidt2018adversarially} showed that adversarially robust models need more training data compared to their standard counterpart. They considered a Gaussian mixture model similar to ours and proved that the training of a robust model requires a training set with size $\Omega(d)$ where $d$ is the dimension of the data, whereas the standard model only needs a constant number of data points. \citet{bubeck2019adversarial} studied a binary classification problem under a statistical query setting and showed that to train a robust classifier one needs exponentially (in dimension $d$) many queries, while only polynomially many to train a standard classifier. The main difference between their work and our work is that we quantify the training dynamic in terms of the size of the training set. Very recently, \citet{javanmard2020precise} precisely characterized the trade-off of standard/robust accuracy under the linear regression setting. \citet{raghunathan2019adversarial} gave empirical evidence that adversarial training could hurt the standard accuracy, despite its improvement on robustness. The PAC-learning setting has also been studied by several authors~\citep{cullina2018pac,diochnos2019lower,montasser2020efficiently}. \citet{cullina2018pac} provided a polynomial (in the VC dimension) upper bound for the sample complexity, while \citet{diochnos2019lower} gave a lower bound for the sample complexity which is exponential in the dimension of the input. 

The strength of the adversary is crucial in the adversarial training. Theoretically, \citet{dohmatob2019generalized} showed that a classifier with high standard accuracy can inevitably be fooled by a strong adversary. Empirically, \citet{papernot2016towards} and \citet{tsipras2018robustness} found that a strong adversary can drive down standard accuracy for robust models. \citet{ilyas2019adversarial} found that the adversarial training tends to learn non-robust features and omit robust ones if the adversary is too strong. 

The double descent phenomenon has been studied by several authors. \citet{belkin2019reconciling,belkin2019two} and \citet{mei2019generalization} provably showed the existence of double descent curves for the generalization error. However, we would like to remark that the double descent curve they considered is in terms of the number of parameters (model complexity), while ours is sample-wise. Empirically, \citet{nakkiran2019deep} also discovered a sample-wise double descent phenomenon.

\section{Preliminaries}

Throughout this paper, let $ [n] $ be a shorthand notation for $ 
\{1,2,\dots,n\} $. Assume the data point $(x,y) $ consists of the input variable $x$ and label $y$, and $(x,y)$ is generated from some distribution $\cD$. Denote the loss function by $\ell(x,y;w)$ and the robust classifier is defined as follows~\cite{goodfellow2015explaining, madry2017towards}: 
\begin{equation}\label{eq:wrob_general}
    \begin{split}
    \wrob_n = {}& \argmin_{w\in \Theta}  \sum_{i=1}^n
	\max_{\tilde{x}_i\in B^\infty_{x_i}(\eps)}
	\ell(\tilde{x}_i, y_i ;w) \,,
    \end{split}
\end{equation}
where $\Theta$ is the parameter space and $B^\infty_{x}(\eps)\triangleq \{ \tilde{x}\in \bR^d| \| \tilde{x}-x \|_\infty \le \eps  \}$ is an $\ell^\infty$ ball centered at $x$ with radius $\eps$. The radius $\eps$ characterizes the strength of the adversary. A larger $\eps$ means a stronger adversary. This robust classifier minimizes the robust loss, or equivalently, maximizes the robust reward (i.e., negative loss).

The generalization error of the robust classifier is given by
\begin{equation}\label{eq:generalization_err}
	\begin{split}
	L_n 
	={} &
	\bE_{\{(x_i,y_i)\}_{i=1}^n\stackrel{\textnormal{i.i.d.}}{\sim} 
		\dgau} \left[ \bE_{(x,y)\sim \dgau}[\ell(x,y;\wrob_n)]\right]\,,
	\end{split}
	\end{equation} where the inner expectation is over the randomness of the test data point and the outer expectation is over the randomness of the training dataset. The test and training data are assumed to be independently sampled from the same distribution. The generalization error can be interpreted as the expected loss of the robust model over standard/unperturbed test data.

\section{Theoretical Results}\label{sec:main}

In this section we study two different binary classification models. In \cref{sec:linear_loss}, we analyze the Gaussian mixture model under linear loss and prove the existence of three possible regimes (weak, medium and strong adversary regimes), in which more training data can help, double descend, or hurt generalization of the adversarially trained model, respectively. In \cref{sec:manhattan}, we construct a model called the Manhattan model that enables us to analyze the 0-1 loss and prove that analogous weak and strong adversary regimes also exist under a different loss function.%

\subsection{Gaussian Mixture with Linear Loss}\label{sec:linear_loss}

\begin{figure}[tb]
    \centering
    \begin{subfigure}[b]{0.21\textwidth}
		\includegraphics[width=\textwidth]{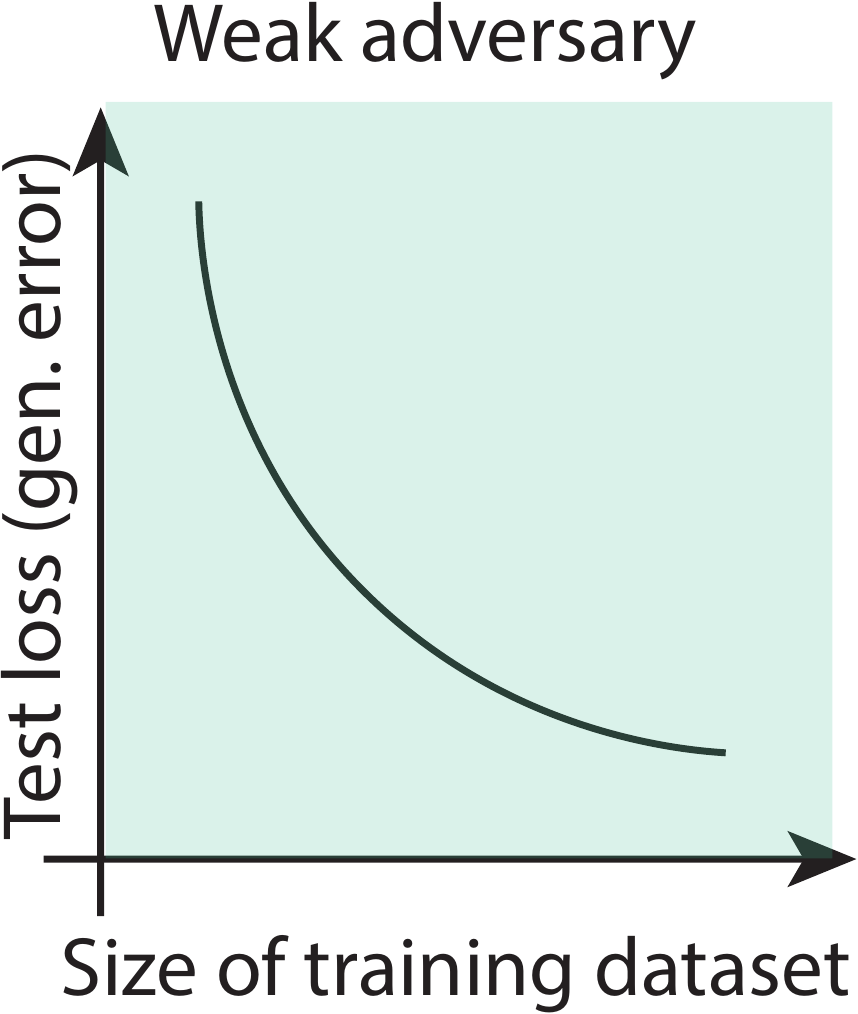}
		\caption{Weak adversary}
		\label{fig:theorem_schematic_weak}
	\end{subfigure}\hspace{0.09\textwidth}%
	\begin{subfigure}[b]{0.21\textwidth}
		\includegraphics[width=\textwidth]{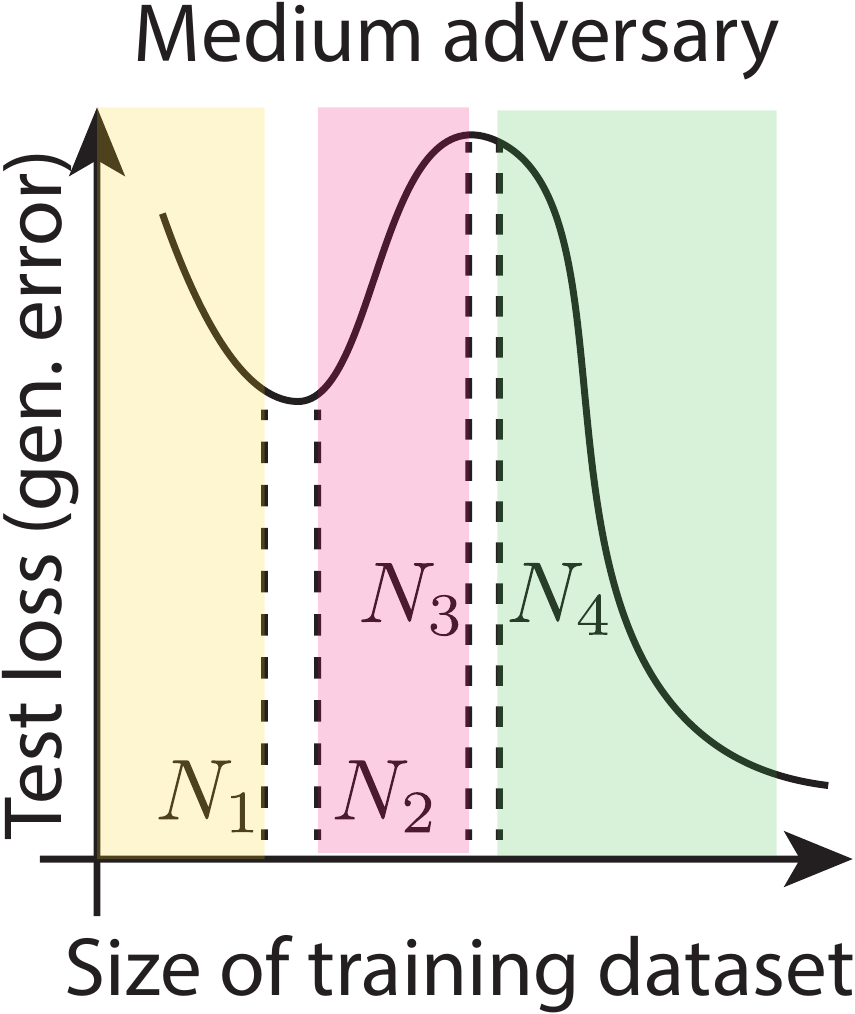}
		\caption{Medium adversary}
		\label{fig:theorem_schematic_medium}
	\end{subfigure}\hspace{0.09\textwidth}%
	\begin{subfigure}[b]{0.21\textwidth}
		\includegraphics[width=\textwidth]{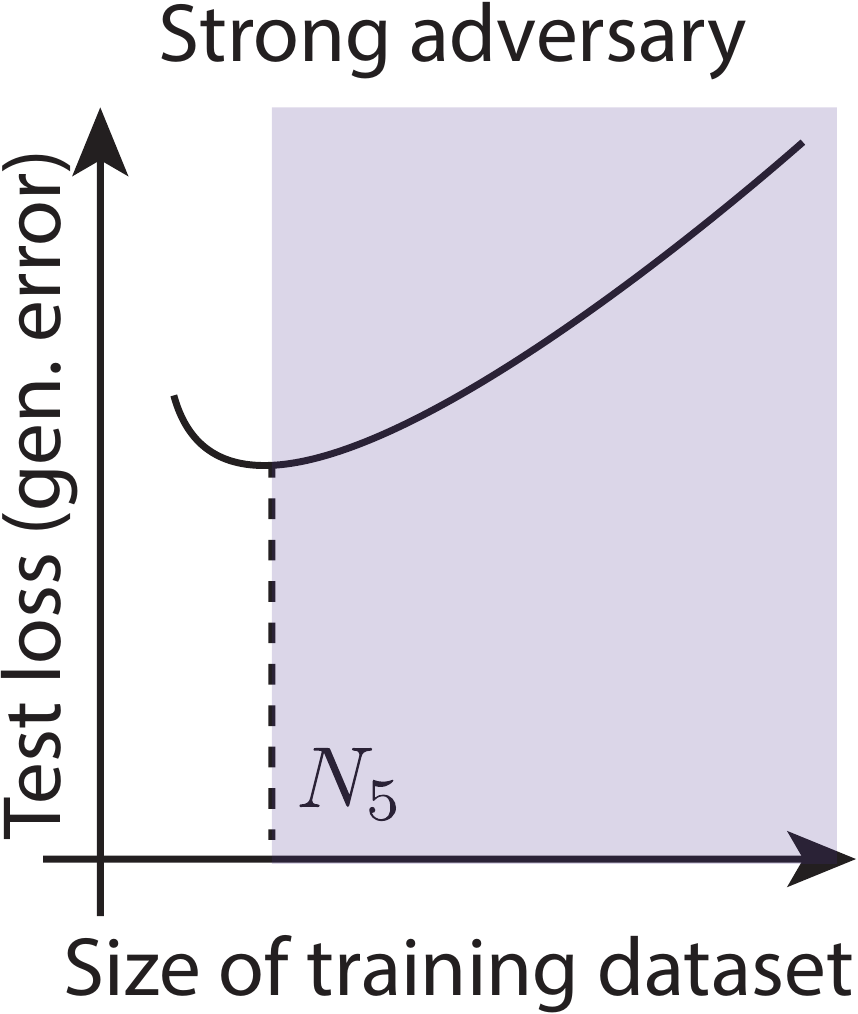}
		\caption{Strong adversary}
		\label{fig:theorem_schematic_strong}
	\end{subfigure}
     \caption{This figure illustrates the three adversary regimes (i.e., weak, medium, and strong) and the corresponding results of \cref{thm:gaussian}. In the weak adversary regime, more training data always improves generalization. The medium adversary regime exhibits a double descent curve. 
     When the size of training set $n\le N_1$ (the initial stage), more training data improves the generalization; when $N_2<n<N_3$ (the intermediate stage), generalization is hurt by more data; when $n\ge N_4$, more data helps with generalization again. 
     In the strong regime, generalization deteriorates with more data when the size of training size is sufficiently large. 
     }
     \label{fig:theorem_schematic}
 \end{figure}

In this subsection, we consider the Gaussian mixture model with linear loss. More specifically, the distribution for the data generation is specified by $ y\sim \unif(\{\pm 1\}) $ and $ x\mid y\sim \cN(y\mu, \Sigma) $, where $ \mu(j)\ge 0 $ for all $j\in [d]$ and $\Sigma=\diag(\sigma^2(1),\sigma^2(2),\dots,\sigma^2(d)) $. In the remaining parts we denote this distribution by $(x,y) \sim \dgau$.
We consider the linear loss $\ell(x,y;w) = -y \langle w,x\rangle$ and we set the constraint set as $ w\in \Theta = \{ w\in \bR^d|  \|w\|_\infty\le W \} $ where $W$ is a positive constant, similar to~\citep{chen2020more,yin2019rademacher,khim2018adversarial}.
In this setting, by \eqref{eq:wrob_general} the robust classifier is
\begin{equation}\label{eq:wrob_linear_loss}
    \begin{split}
    \wrob_n = {}& \argmin_{\|w\|_\infty\le W}  \sum_{i=1}^n
	\max_{\tilde{x}_i\in B^\infty_{x_i}(\eps)}
	\left( -y_i\langle w,\tilde{x_i}\rangle \right) = \argmax_{\|w\|_\infty\le W} \sum_{i=1}^n 
	\min_{\tilde{x}_i\in B^\infty_{x_i}(\eps)}
	y_i\langle w,\tilde{x_i}\rangle\,.
    \end{split}
\end{equation} 
We study how the generalization error of the robust model evolves as the size of the training dataset changes, i.e., the dependence of $L_n$ on $n$.
By \eqref{eq:generalization_err} the generalization error of the robust classifier under linear loss is given by
\begin{equation}\label{eq:generalization_err_linear_loss}
	\begin{split}
	L_n 
	={} &
	\bE_{\{(x_i,y_i)\}_{i=1}^n\stackrel{\textnormal{i.i.d.}}{\sim} 
		\dgau} \left[ \bE_{(x,y)\sim \dgau}[-y\langle 
	\wrob_n,x\rangle]\right]\,.
	\end{split}
	\end{equation}
 For the Gaussian classification problem under the linear loss, we identify that the behavior of $L_n$ exhibits a phase transition which is determined by the strength of the adversary. Our main result is summarized by \cref{thm:gaussian}.

\begin{theorem}[Proof in \cref{sec:proof_linear_loss}]\label{thm:gaussian}
Given $n$ i.i.d.\,training data points $(x_i , y_i) \sim \dgau$, if the robust 
classifier is defined by \eqref{eq:wrob_linear_loss} and its generalization 
error is defined by \eqref{eq:generalization_err_linear_loss}, then there exist 
$0 < \delta_1 < \delta_2 < 1$, such that
\begin{enumerate}[label=(\alph*),nosep]
    \item If $0 < \eps < \delta_1 \cdot \min_{j \in [d]} \mu(j)$, then $L_n < L_{n-1}$ for all $n$. That is, the loss $L_n$ monotonically decreases as the number of training points $n$ increases. %
    \label{it:weak_regime}
    \item If $\delta_2 \cdot \max_{j \in [d]} \mu(j)<\eps <\min_{j \in [d]} \mu(j)$, and we further %
    assume that $\frac{\mu (j)}{\sigma (j)}$ is the same for all $j$, then there exist $N_1 < N_2 < N_3 < N_4$ such that 
    \begin{equation*}
        \begin{split}
            L_n  \begin{cases}
            < L_{n-1} & \quad \textnormal{for} \ 0<n \leq N_1  \,,
            \\ > L_{n-1} & \quad \textnormal{for} \ N_2 < n < N_3 \,,
            \\ < L_{n-1} & \quad \textnormal{for} \ N_4 \leq n \,.
            \end{cases}
        \end{split}
    \end{equation*} \label{it:medium_regime}
    \item If $\max_{j \in [d]} \mu(j) \leq \eps $, then there exists $N_5$ such that $L_n > L_{n-1}$ for all $n> N_5$. \label{it:strong_regime}
\end{enumerate}
\end{theorem}

\cref{thm:gaussian} verifies the existence of three possible regimes during the commonly used adversarial training procedure and gives conditions for when the phase transition between these regimes will take place. Part \ref{it:weak_regime} identifies the weak regime, showing that when the strength of the adversary $\eps$ is small compared to the signal $\mu$, the generalization error decreases as the size of the training dataset increases. In this regime, the generalization benefits from the use of a large training set. This regime is illustrated by \cref{fig:theorem_schematic_weak}, where the curve is always decreasing.

However, as the adversary becomes stronger, we reach the medium regime and things change. Part \ref{it:medium_regime} proves the existence of a double descent curve for the generalization error. It shows that when $\eps$ becomes larger and approaches the signal in magnitude, the generalization error will first decrease as more training data is used. Surprisingly, once it reaches a certain point, it will start increasing as we feed more data. This increasing stage continues until the dataset size reaches some threshold $N_2$ and then the error will decrease again. The medium adversary regime is illustrated by \cref{fig:theorem_schematic_medium}, where the three stages are marked by three different colored areas. 

If the adversary's strength reaches the signal level or becomes even stronger, then for all sufficiently large $n$, the generalization error monotonically increases as the size of training set increases. This strong regime is described in part \ref{it:strong_regime} of \cref{thm:gaussian} and illustrated by \cref{fig:theorem_schematic_strong}. Note that despite the decreasing stage near the very beginning, the loss keeps going up after the threshold $N_5$. 

Furthermore, we see that in the medium regime, the length of the increasing stage is given by  $N_3-N_2$, according to part \ref{it:medium_regime} of \cref{thm:gaussian}. We would like to remark that the model can have an arbitrarily long increasing stage, which depends on the adversary's strength. To better interpret this idea and the meaning behind \cref{thm:gaussian}, we consider the following special case where $\mu(j) = \mu_0$ and $\sigma (j) = \sigma_0$ for all $j\in [d]$.  %
In this special case, it can be shown that in the medium regime, as $\eps$ approaches the signal strength $\mu_0$, the increasing stage grows and can be arbitrarily long.

\begin{cor}[Proof in \cref{sec:proof_linear_loss}]\label{cor:gaussian}
Under the same assumption as \cref{thm:gaussian} and further assuming that $\mu(j) = \mu_0$ and $\sigma (j) = \sigma_0$ for all $j\in [d]$, we have
\begin{enumerate}[label=(\alph*),nosep]
    \item If $0 < \eps < \delta_1 \mu_0$, then $L_n < L_{n-1}$ for all $n$. \label{it:cor_weak_regime}
    \item If $\delta_2 \mu_0<\eps < \mu_0$, then there exist $N_1 (\eps) < N_2 (\eps) $ such that 
    \begin{equation*}
        \begin{split}
            L_n  \begin{cases}
            < L_{n-1} & \quad \textnormal{for} \ 0<n \leq N_1  \,,
            \\ > L_{n-1} & \quad \textnormal{for} \ N_1 < n < N_2 \,,
            \\ < L_{n-1} & \quad \textnormal{for} \ N_2 \leq n \,,
            \end{cases}
        \end{split}
    \end{equation*} and $\lim_{\eps \to \mu_0^-} N_2 (\eps) - N_1 (\eps) = + \infty$.
    \label{it:cor_medium_regime}
    
    \item If $\mu_0 \leq \eps $, then there exists $N_3 (\eps) $ such that $ L_n > L_{n-1} $ for all $n> N_3$. \label{it:cor_strong_regime}
\end{enumerate}
\end{cor}

Part \ref{it:cor_weak_regime} and \ref{it:cor_strong_regime} of 
\cref{cor:gaussian} are a re-statement of corresponding parts of 
\cref{thm:gaussian} in the simplified setting. Part \ref{it:cor_medium_regime} 
additionally states that as $\eps$ increases towards $\mu_0$, the length of the increasing stage goes to infinity. In this setting, the three regimes are marked by the thresholds $\delta_1 \mu_0$, $\delta_2 \mu_0$ and $\mu_0$. 

\begin{figure}[htb]
	\begin{subfigure}[b]{0.33\textwidth}
		\includegraphics[width=\textwidth]{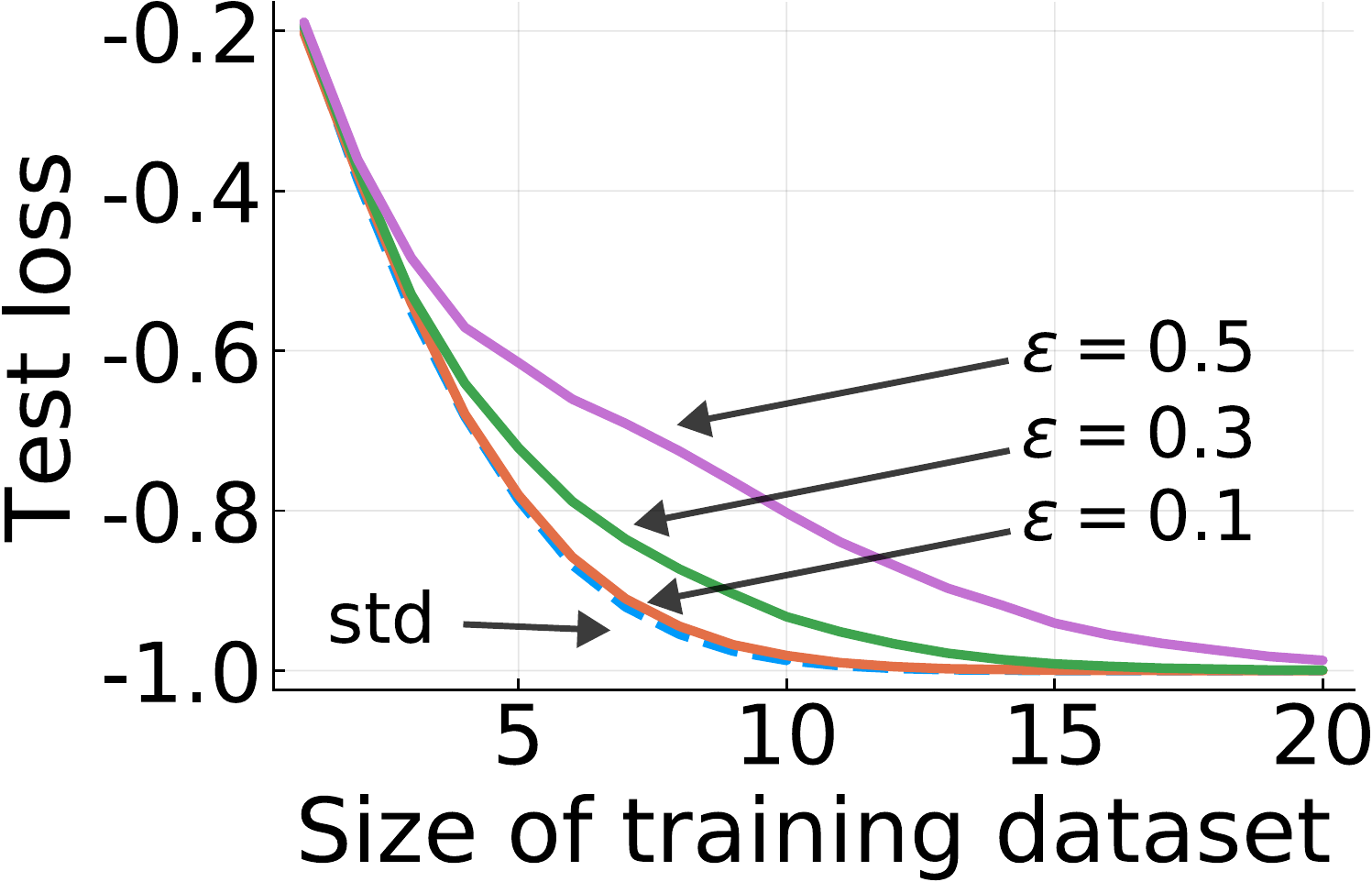}
		\caption{Weak adversary}
		\label{fig:weak_linear_loss}
	\end{subfigure}
	\begin{subfigure}[b]{0.33\textwidth}
		\includegraphics[width=\textwidth]{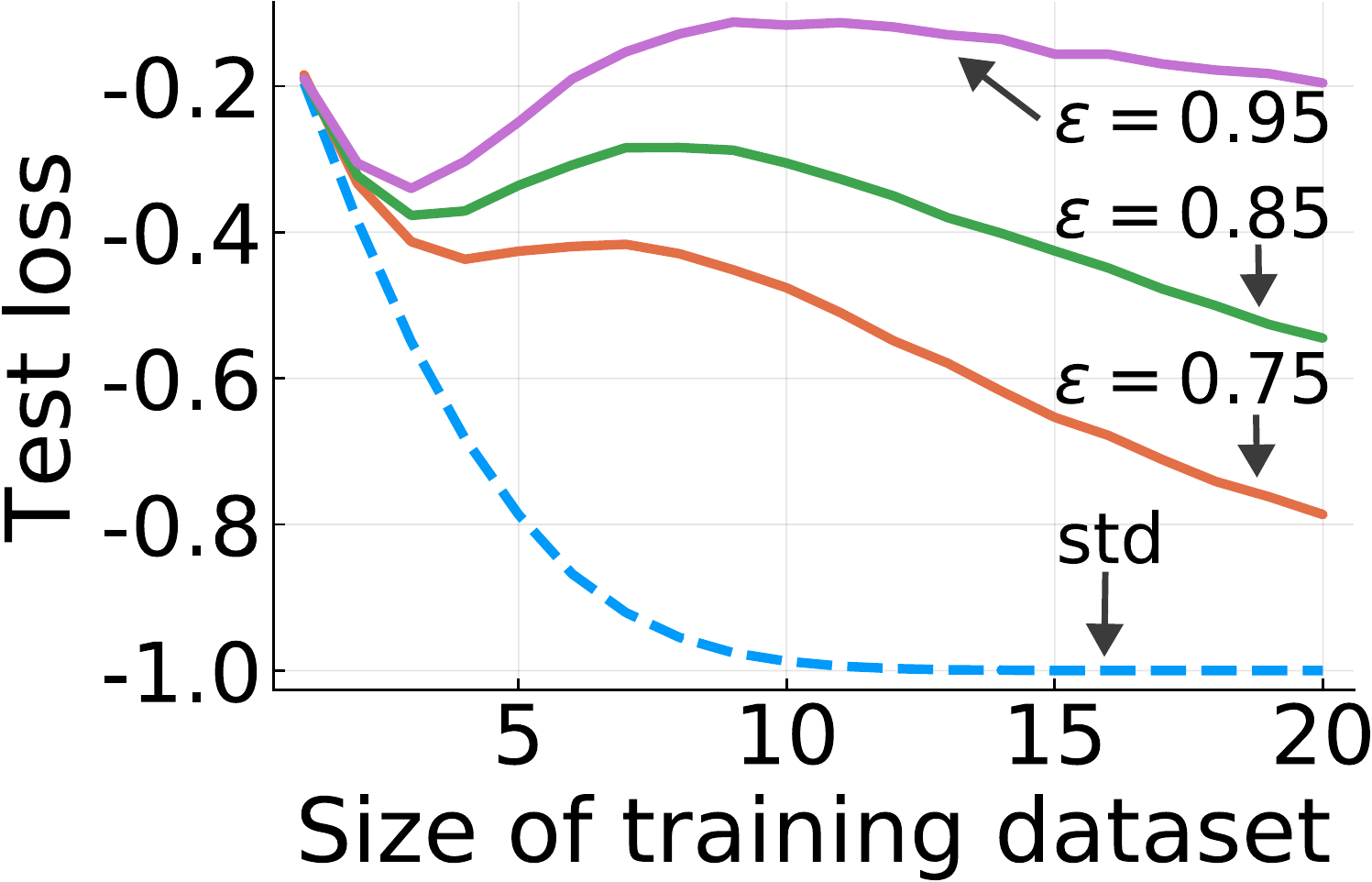}
		\caption{Medium adversary}
		\label{fig:medium_linear_loss}
	\end{subfigure}
	\begin{subfigure}[b]{0.33\textwidth}
		\includegraphics[width=\textwidth]{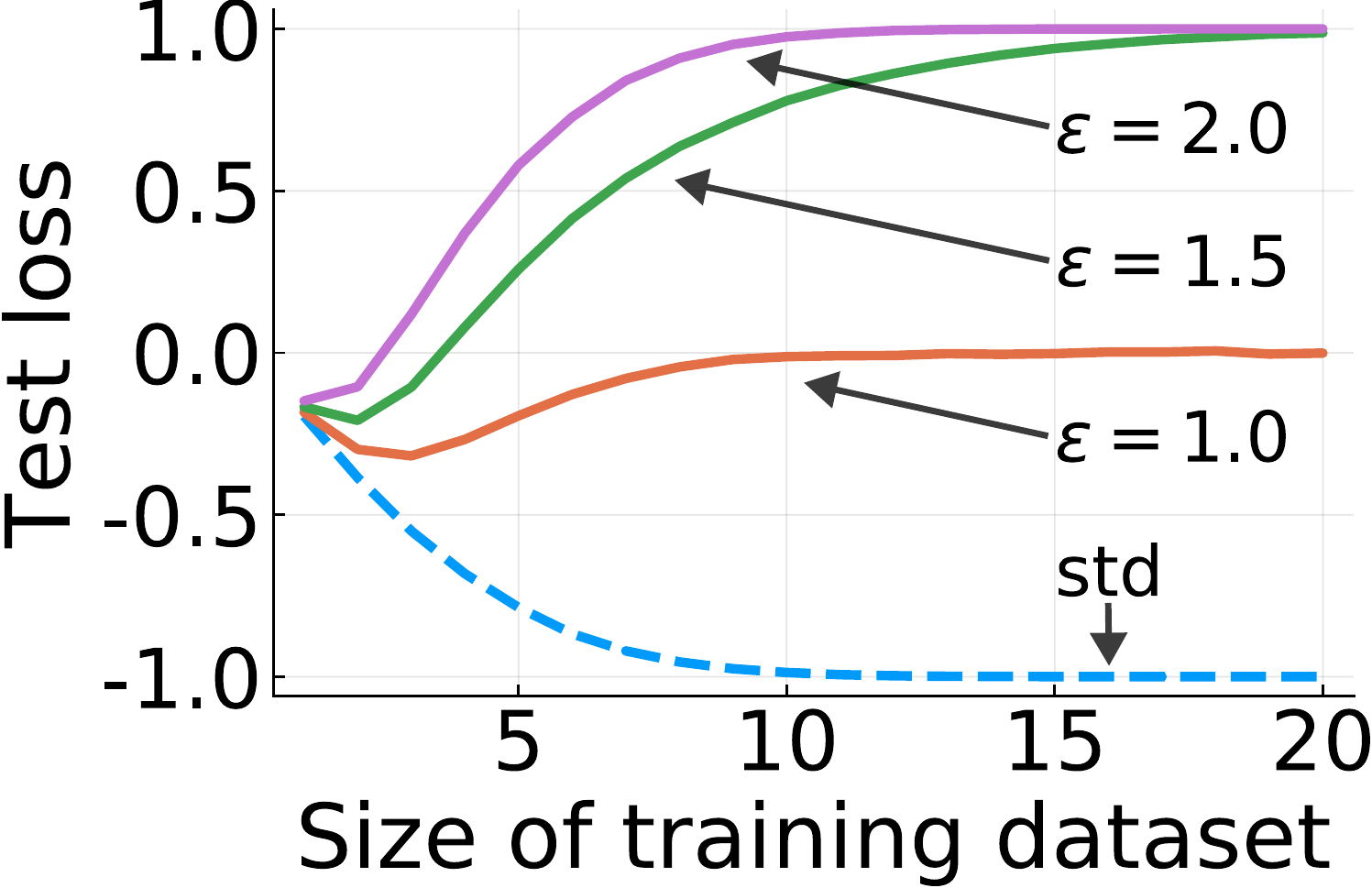}
		\caption{Strong adversary}
		\label{fig:strong_linear_loss}
	\end{subfigure}
	\caption{The test loss versus the size of the training dataset under the linear loss and the one-dimensional ($d=1$) Gaussian data generation model described in \cref{sec:linear_loss}. The parameters of the Gaussian data model are set as follows: $\mu_0=1$ and $\sigma_0=2$. In each plot, the solid curves correspond to robust models and the dashed curve corresponds to the standard model.}
	\label{fig:linear_loss}
\end{figure}

\cref{fig:linear_loss} illustrates the behavior of the generalization error in this simplified setting. In the simulation  we set the parameters as $d=1$, $ \mu_0=1 $ and $ \sigma_0=2 $ (for all three plots). \cref{fig:weak_linear_loss} shows the weak adversary regime. We see that the generalization error maintains a decreasing trend when $\eps$ is as large as half the signal strength. 
In \cref{fig:medium_linear_loss}, it is clear that the generalization error has a double descent curve. At first there is a decreasing stage, which is followed by an increasing stage. Also observe that as $\eps$ becomes larger, the error increases faster during the increasing stage. The error will finally start decreasing as the size of training dataset reaches the second decreasing stage. On the contrary, in the strong adversary regime, the increasing stage lasts forever and the error keeps increasing no matter how much data is provided, as illustrated by \cref{fig:strong_linear_loss}.

\subsection{Manhattan Model}\label{sec:manhattan}

In general, the 0-1 loss is mathematically intractable for most data models and computationally prohibitive to optimize in practice. With this in mind, we introduce a conceptual classification model that we call the Manhattan model. Note that this model is highly simplified and thus unlikely to be suitable for modeling real-world problems. Instead, the purpose of the Manhattan model is to allow a mathematical study of the 0-1 loss, and thus provide a springboard for the study of 0-1 loss in more complicated models.

\begin{figure}[htb]
    \begin{center}
        \includegraphics[width=0.5\linewidth]{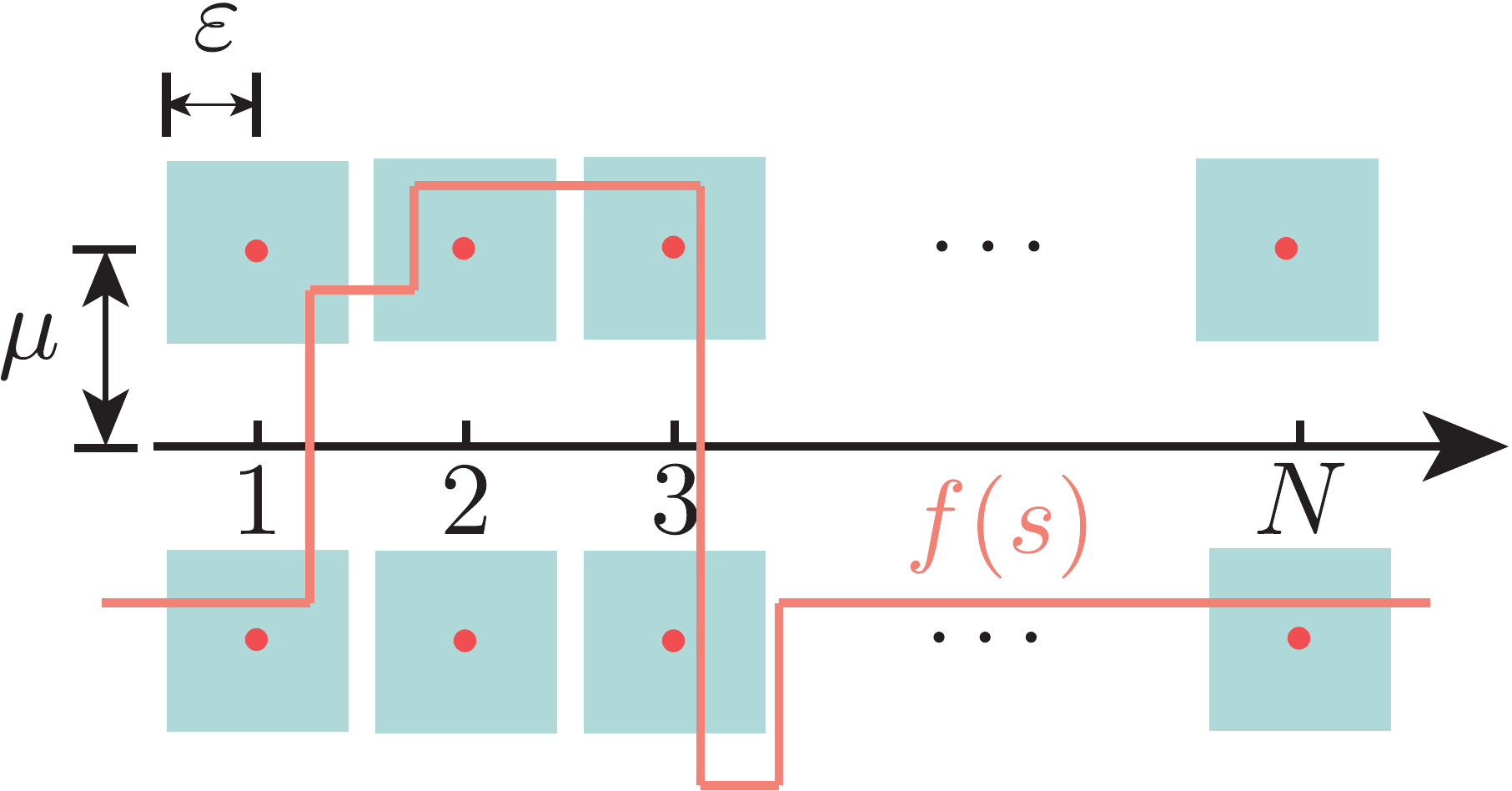}
    \end{center}
    \caption{An illustration of the Manhattan model. The data is in the $\bR^2$ plane and the support is 2$N$ points on opposite sides of the axis. The distance between any point and the axis is $\mu$, and the shaded square denotes the $\eps$ perturbation. The red curve shows a possible classifier.}
    \label{fig:tiling}
\end{figure}

We start by describing the data distribution. Assume we have data points $(x,y) \in \bR^2 \times \{\pm 1\}$, where the support of $x$ is given as $x=(s,t) \in \{ (i,y\mu)$ \ : \  $i\in [N], \ y \in \{ \pm 1 \} \}$, where $0 < \mu < 1/4$. In other words, every data point $(x,y)$ consists of a positive or negative label $y$ and a point on the 2-D plane $x = (s,t)$ where $s$ is an integer between 1 and $N$ and $t$ is either $\mu$ or $-\mu$ depending on whether the label $y$ is $+1$ or $-1$. Thus, the support consists of exactly $2N$ points with half in the positive class and half in the negative class. The data is uniformly sampled from these $2N$ points and this distribution is denoted by $\cD_{2N}$.

Next, we consider a conceptual classifier of the form of a step function over the 2-D $(s,t)$-plane. That is, a classifier is defined by a function $t = f(s)$ such that $f\in F$ where
\begin{equation}
    F = \left\{ f : \   f(s) = \sum_{j=1}^M \alpha_j \ind[s \in I_j] ,\ M \in \mathbb{N},\ \alpha_j \in \bR,\  I_j \subseteq \bR \ \textnormal{are intervals} \right\}.
\end{equation} A point $x=(s,t)$ is classified $+1$ if $t > f(s)$ and $-1$ if $t < f(s)$. If $t = f(s)$, then $x$ is classified as either $+1$ or $-1$ uniformly at random. \cref{fig:tiling} illustrates the support of the data distribution, as well as a possible classifier $f(s)$. 

Since we consider the 0-1 loss, one can note that for this data model, there can be infinitely many classifiers. For example, $f(s) = c$ can attain 100\% standard accuracy for all $c \in (-\mu, \mu)$. Therefore, we add an infinitesimal $\ell_1$ penalty to the 0-1 loss for the purpose of tie-break, i.e., making the minimizer unique. This penalized 0-1 loss of a classifier $f(\cdot)$ can then be written as $H\left(-y (t - f(s)) \right) + \lambda \|f\|_1$ with $\lambda\to 0$. For a given training set $\{ (x_i,y_i):i \in [n]\}$, We define the robust classifier over this training set as 
\begin{equation}\label{eq:frob_brick_wall}
    \frob_n \in \lim_{\lambda \to 0^+}\left[ S(\lambda)\triangleq \argmin_{f \in F} \sum_{i=1}^n \max_{\|\tilde{x}_i-x_i\|_\infty < \eps} H\left(-y_i (\tilde{t}_{i} - f(\tilde{s}_{i})) \right)  + \lambda \|f\|_1\right]\,,
\end{equation}
where $H(s) = \ind[s>0]+\frac{1}{2}\ind[s=0]$ is the Heaviside step function.
Note that the RHS of \cref{eq:frob_brick_wall} is the limit of a sequence of sets. This slight abuse of notation is justified by the following \cref{lem:frob_well_defined}, which shows for all sufficiently small $\lambda$, the set $S(\lambda)$ remains fixed. We define the set of candidate classifiers without the penalty as
\begin{equation*}
    S = \argmin_{f \in F} \sum_{i=1}^n \max_{\|\tilde{x}_i-x_i\|_\infty < \eps} H\left(-y_i (\tilde{t}_{i} - f(\tilde{s}_{i})) \right),
\end{equation*} and we have the following lemma. 
\begin{lemma}[Proof in \cref{sec:proof_frob_equivalent_definition}]\label{lem:frob_well_defined}
For all sufficiently small $\lambda>0$ and for any $\eps < 1/2$, the set $S(\lambda)$ defined by \cref{eq:frob_brick_wall} is equivalent to the following set which is nonempty
\begin{equation}
    \begin{split}
       S_2 &  =  S \cap \argmin_{f\in S} \|f\|_1.
    \end{split}
\end{equation}
\end{lemma}\cref{lem:frob_well_defined} shows that by picking a small enough $\lambda$, the minimizers with $\ell_1$ penalty actually coincide with the minimizers under 0-1 loss with the smallest $\ell_1$ norm. Therefore, $\lim_{\lambda \to 0^+} S(\lambda) = S_2$ and we can write $\frob \in S_2$ as an equivalent definition of the robust classifier to \cref{eq:frob_brick_wall}.

The generalization error of $f_n^{\text{rob}}$ is then given by
\begin{equation}\label{eq:generalization_err_brick}
	\begin{split}
	L_n 
	={} &
	\bE_{\{(x_i,y_i)\}_{i=1}^n\stackrel{\textnormal{i.i.d.}}{\sim} 
		\cD_{2N}} \left[ \bE_{(x,y)\sim \cD_{2N}} H \left(  -y\left( t - \frob_n(s) \right) \right)  \right]\,,
	\end{split}
\end{equation} where we can get rid of the $\ell_1$ term due to \cref{lem:frob_well_defined}. \cref{thm:brick-wall-test-loss} shows that the generalization error can be zero for all $n$ when $\eps<2\mu$ and can increase with $n$ as $\eps > 2\mu$.

\begin{theorem}[Proof in \cref{sec:proof_brick-wall-test-loss}]\label{thm:brick-wall-test-loss}
Assume the training data $(x_i , y_i) \sim \cD_{2N}$ where $i \in [n]$. For the robust 
classifier defined by \eqref{eq:frob_brick_wall} and its generalization 
error defined by \eqref{eq:generalization_err_brick}, we have
\begin{enumerate}[label=(\alph*),nosep]
    \item If  $ \ 0 < \eps < 2\mu $, then $L_n =0\ $ for all $n$. 
    \label{it:weak_regime_brick_wall}
    \item If $ \ 2\mu < \eps\leq 1/2$, then $L_{n+1} > L_n\ $ for all $n \geq 1$. 
    \label{it:strong_regime_brick_wall}
\end{enumerate}
\end{theorem}

Again, the purpose of the Manhattan model is not to model any real-world problems, but instead to show that adversarial training under a 0-1 loss can also be characterized with weak/strong regimes. More generally, we have now shown that the existence of weak/strong regimes is not solely an artifact of the linear loss used in \cref{sec:linear_loss}, and thus that it may not be surprising to see analogous results for a much broader class of loss functions.

\section{Empirical Results}\label{sec:experiment}

In this section, we empirically study the generalization error of robust models in three settings. 

\subsection{Gaussian Mixture with 0-1 Loss}\label{sec:experiment_gaussian-0-1}

We consider the 1-dimensional Gaussian mixture model with 0-1 loss. The data generation is the same as in \cref{sec:linear_loss} with $d=1$. Here we set $\Theta=\bR$ and the classifier is represented by a real number $ w\in \bR $. That is, a point is classified as positive or negative depending on whether $x$ is greater than or less than $w$. If $x=w$, it is uniformly randomly classified as positive or negative. Given a data point $ (x,y) $, the 0-1 loss of classifier $ w $ is given by $ \ell(x,y;w) = \ind[y(x-w)<0] $. 

We remark that under this setting, the robust classifier is not unique and the set of classifiers is an interval (details in \cref{sec:minimizer_is_interval}). Thus to select a classifier, we consider two tiebreaking methods. One is the agnostic tiebreak, which means the classifier is chosen uniformly at random from the interval. The other is the optimal tiebreak in hindsight, referring to picking the classifier from the interval with the smallest expected test loss. The test loss of a classifier $w$ is given by 
\begin{equation}\label{eq:test-loss-under-0-1-loss}
	\bE_{(x,y)\sim \dgau}[\ind[y(x-w)<0]] = \frac{1}{2} + \frac{1}{2}\left( 
	\Phi\left(\frac{w-\mu}{\sigma}\right) - \Phi\left( \frac{w+\mu}{\sigma} 
	\right) \right) \,,
\end{equation}where $\Phi$ is the CDF of the standard normal distribution. In \cref{sec:0-1-test loss}, we explain that the optimal classifier in hindsight is the one that is closest to $0$ among the interval of classifiers.

\begin{figure*}
	\centering
	\begin{subfigure}[b]{0.23\textwidth}
		\includegraphics[width=\textwidth]{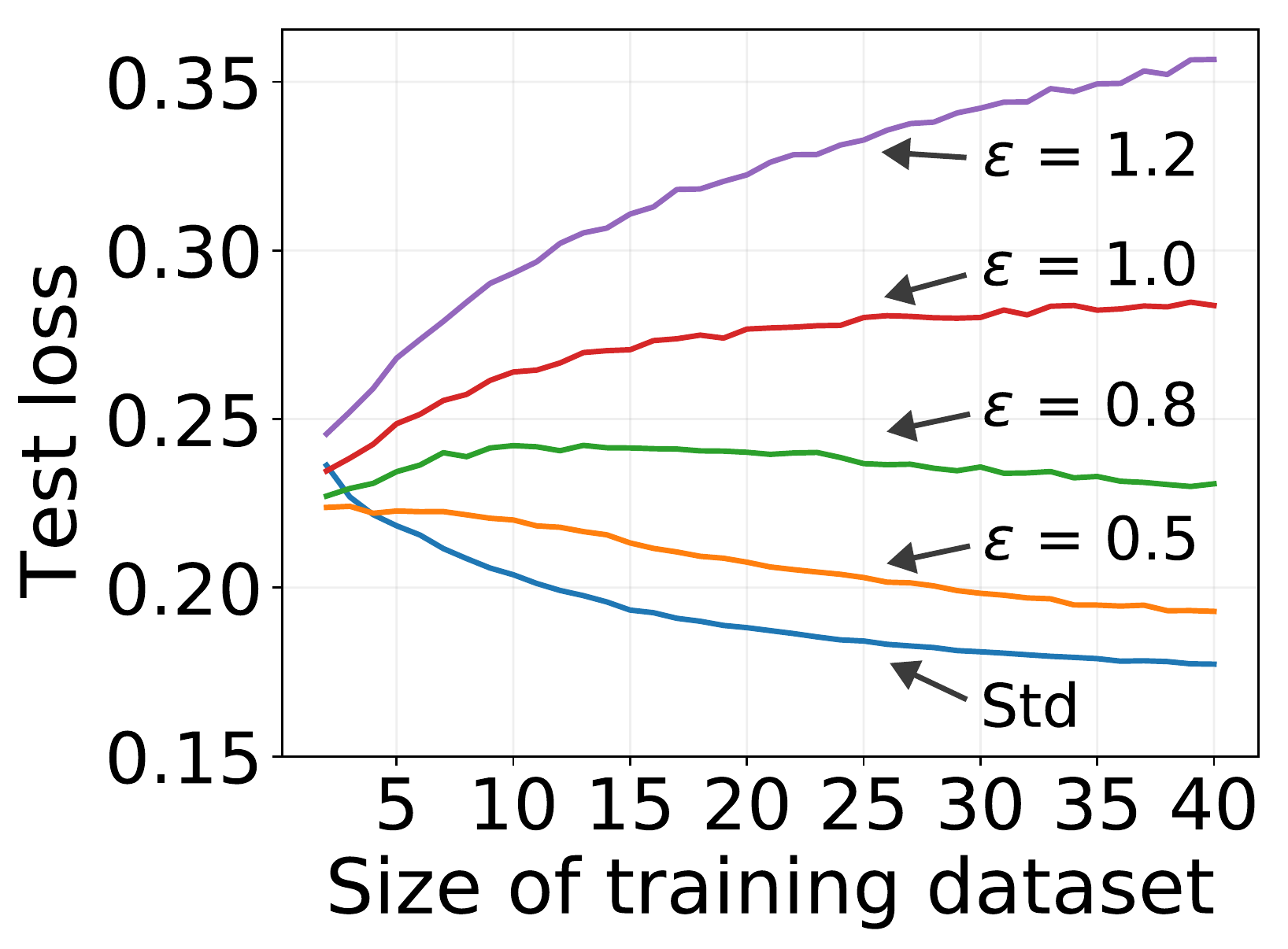}
		\caption{Agnostic tiebreak}\label{fig:bin_agnostic}
	\end{subfigure}
	\begin{subfigure}[b]{0.23\textwidth}
		\includegraphics[width=\textwidth]{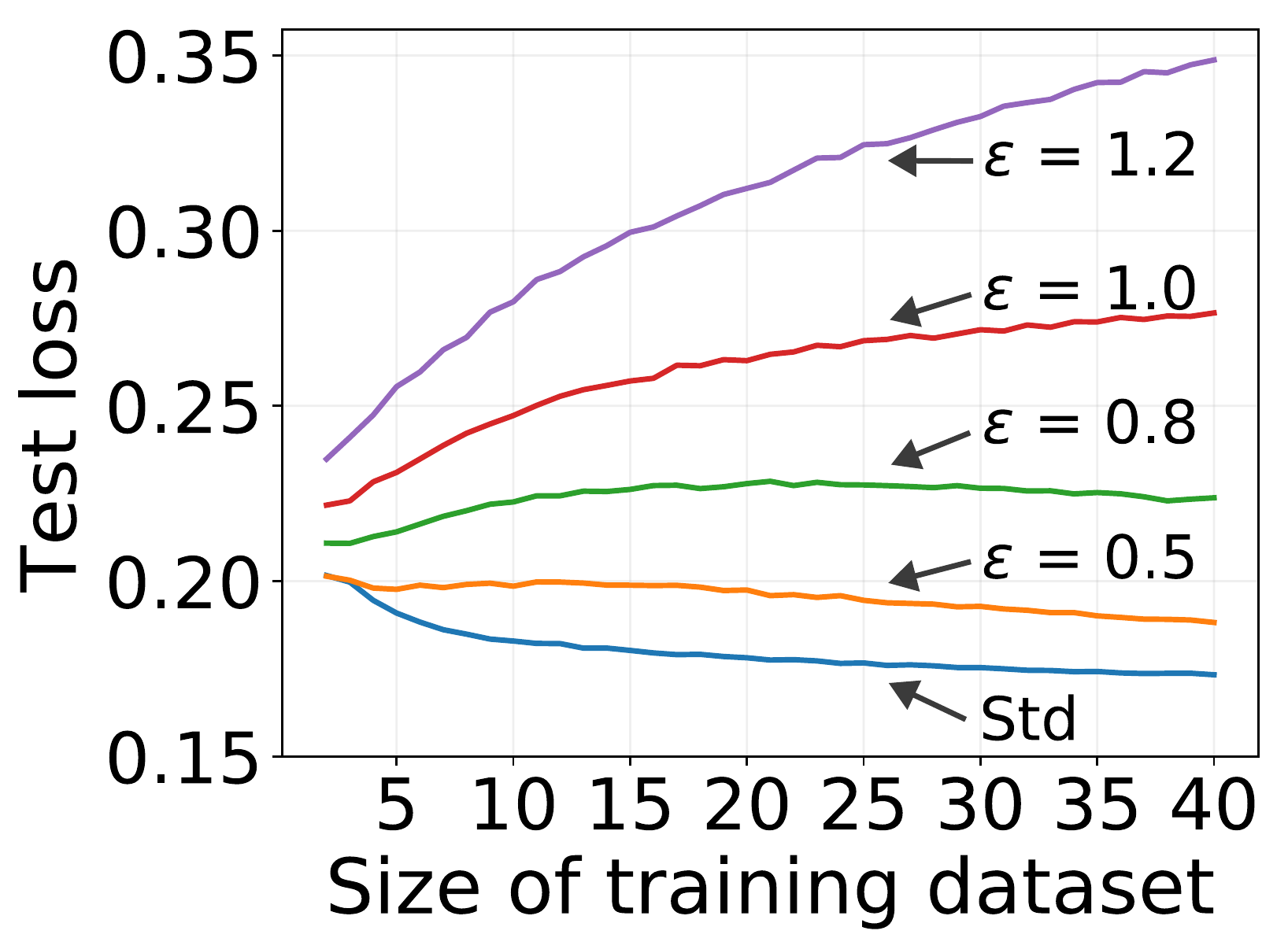}
		\caption{Optimal tiebreak}\label{fig:bin_hindsight}
	\end{subfigure}
	\begin{subfigure}[b]{0.247\textwidth}
		\includegraphics[width=\textwidth]{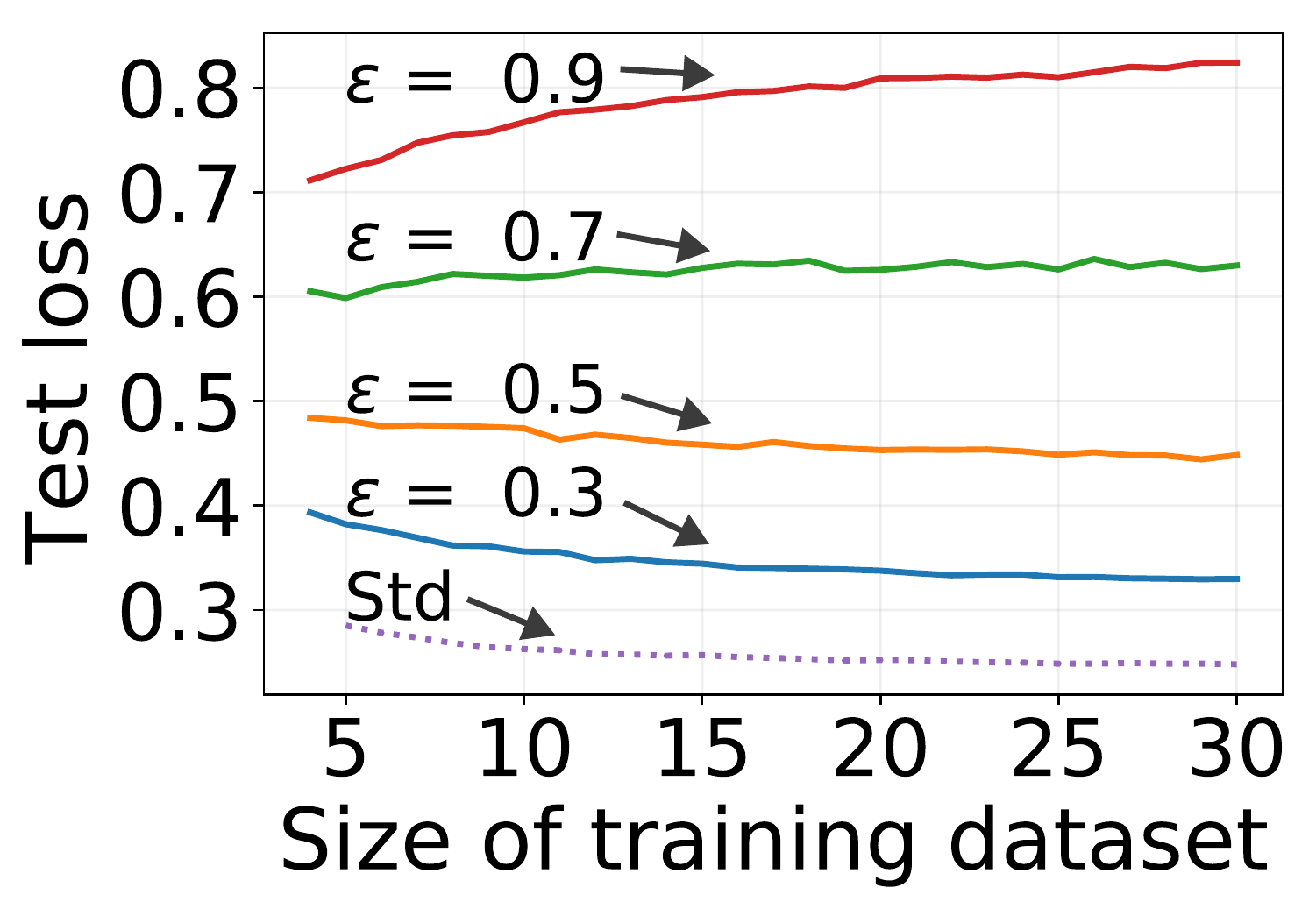}
		\caption{SVM (small $\eps$)}
		\label{fig:weak_svm_hinge_loss}
	\end{subfigure}
	\begin{subfigure}[b]{0.254\textwidth}
		\includegraphics[width=\textwidth]{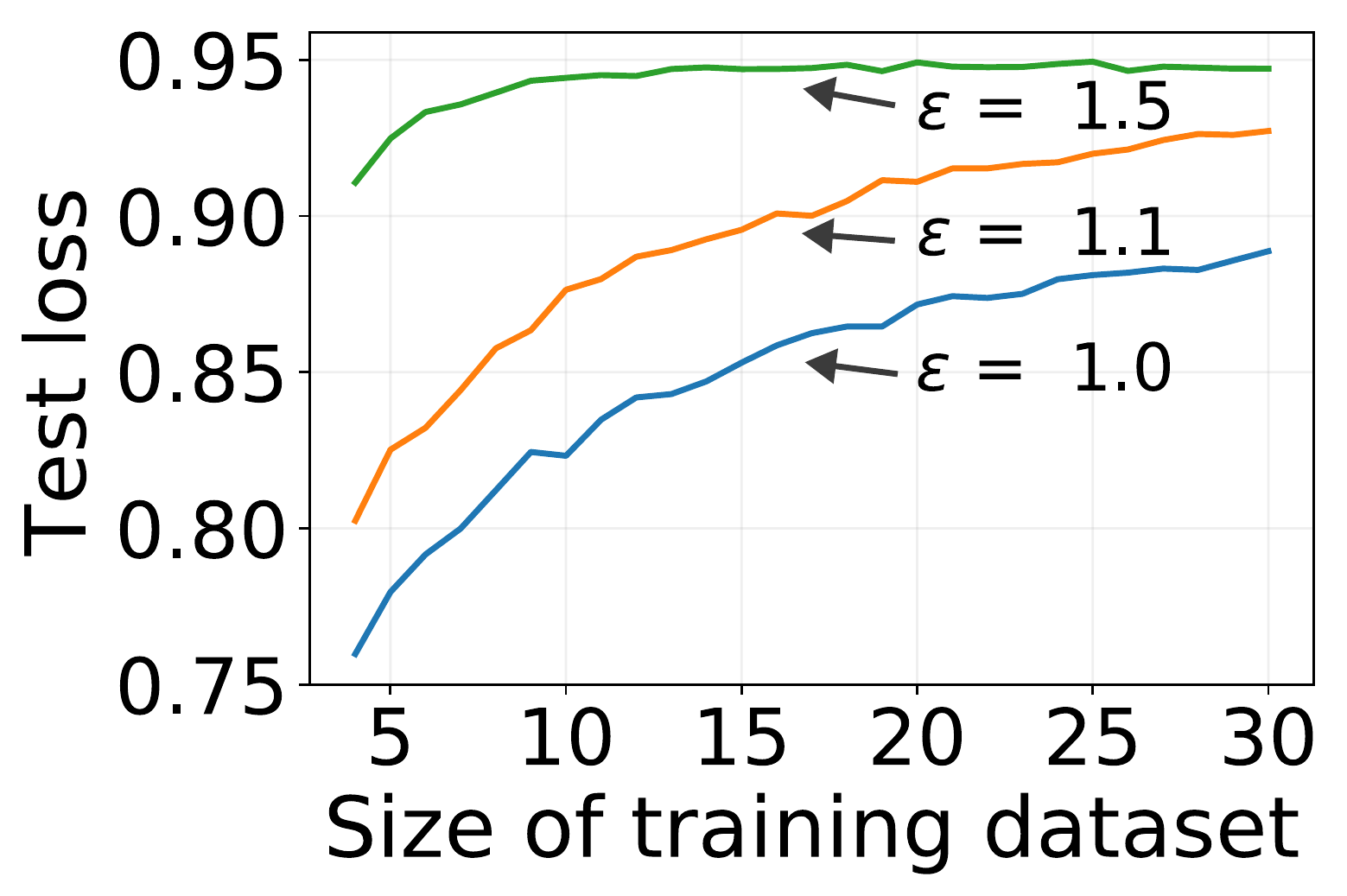}
		\caption{SVM (large $\eps$)}
		\label{fig:strong_svm_hinge_loss}
	\end{subfigure}
	\caption{\cref{fig:bin_agnostic,fig:bin_hindsight} present the test loss vs.\ the size of the training dataset for Gaussian mixture in the 0-1
		loss setting described in \cref{sec:experiment_gaussian-0-1}.  \cref{fig:weak_svm_hinge_loss,fig:strong_svm_hinge_loss} illustrate the test loss vs.\ the size of the training dataset for the support vector machine model described in \cref{sec:experiment_svm}.}
\end{figure*}

\cref{fig:bin_agnostic} and \cref{fig:bin_hindsight} illustrate the test loss versus the size of the training dataset under the agnostic tiebreak and the optimal tiebreak in hindsight. We set $ \mu=\sigma=1 $ and use the same set of values for $\eps$ for both tiebreaking methods. We have three observations. First, the generalization error is increasing in $n$ when $\eps$ is larger than the signal strength. This confirms the existence of the strong adversary regime under the 0-1 loss. Second, for small enough $\eps$ (e.g.\, $\eps \leq 0.5$), the generalization error is decreasing in $n$ (more precisely after $n=3$), thus also confirming a weak adversary regime. For the medium adversary where $\eps$ is in between $0.7$ and $1.0$, the curve has an increasing stage followed by a decreasing stage, which is very similar to what we see in \cref{fig:medium_linear_loss}.

\subsection{Support Vector Machine}\label{sec:experiment_svm}

We study the soft-margin support vector machine with hinge loss (details in \cref{sec:svm-detail}). The dimension $d$ equals 2 and the data is generated as $y \sim \unif(\{\pm 1\})$ and $X \sim \cN(y \mu , I)$ where $\mu = (1,1)^\top$. The results are shown in \cref{fig:weak_svm_hinge_loss} and \cref{fig:strong_svm_hinge_loss}. We find that for small $\eps$ the standard test loss keeps decreasing, while for large $\eps$ it keeps increasing. The curves reveal a transition from the weak to the strong regime as $\eps$ grows, and such transition occurs when $\eps$ is in between 0.5 and 0.7. Note that at $\eps = 0.7$, the test loss increases even though the strength of the adversary is still weaker than the signal level. This may indicate that for more complicated models (such as SVMs), even relatively weaker adversaries can result in situations where more data always increases the test loss.

\subsection{Linear Regression}\label{sec:experiment_lr}

\begin{figure}[thb]
    \centering
	\begin{subfigure}[b]{0.239\textwidth}
		\includegraphics[width=\textwidth]{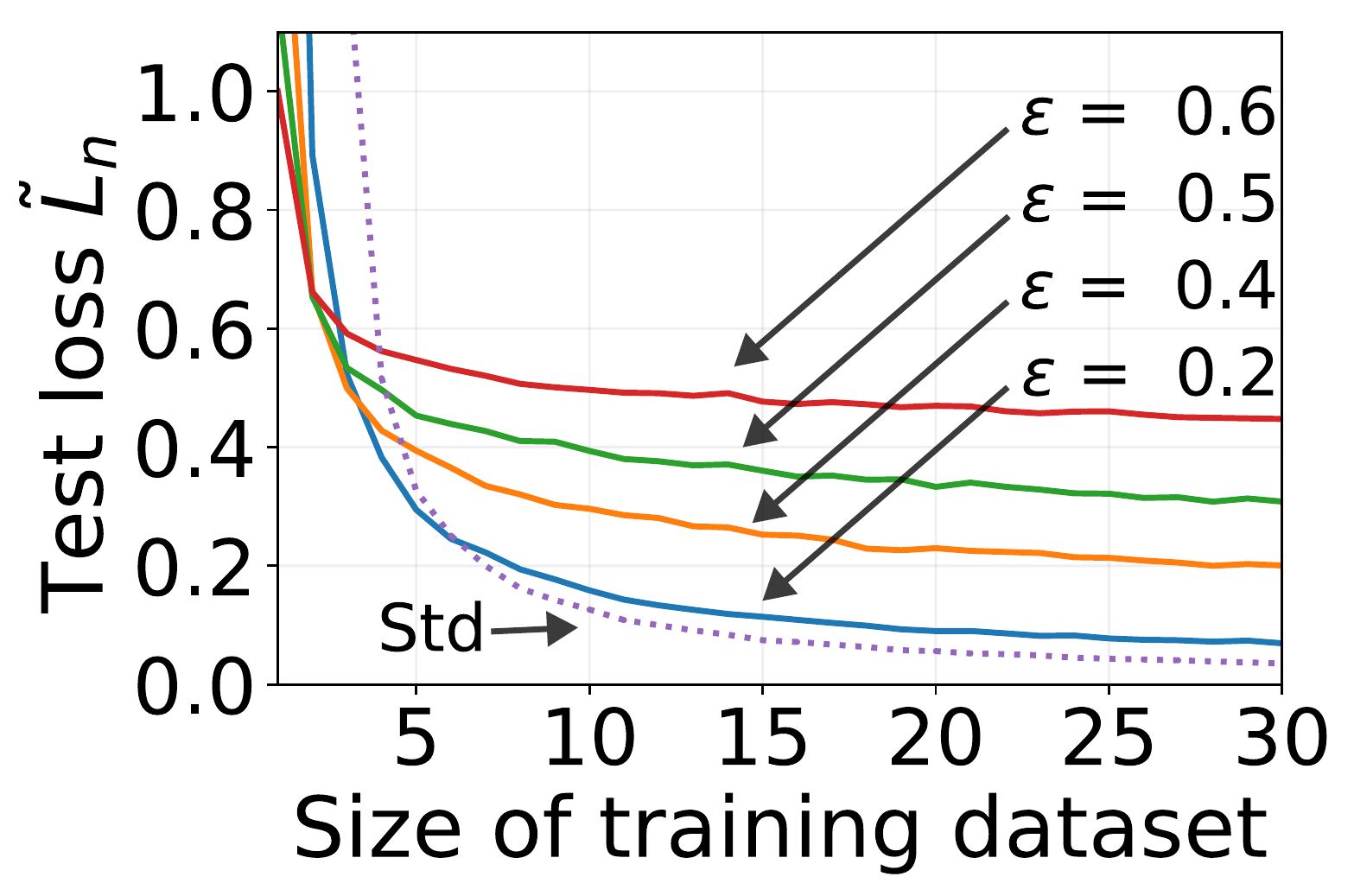}
		\caption{Gaussian, small $\eps$}
		\label{fig:lr_gaussian_weak}
	\end{subfigure}
	\begin{subfigure}[b]{0.239\textwidth}
		\includegraphics[width=\textwidth]{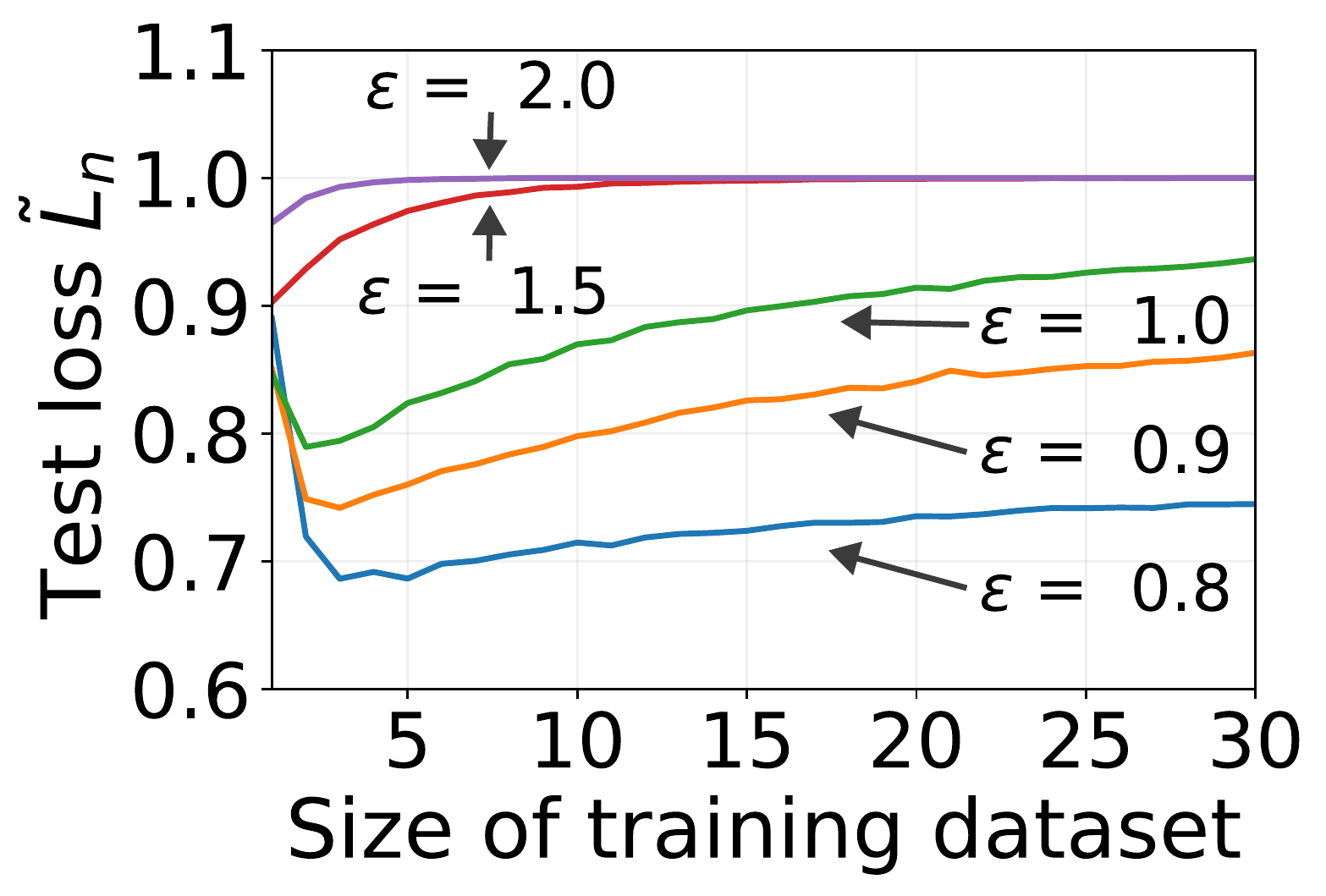}
		\caption{Gaussian, large $\eps$}
		\label{fig:lr_gaussian_strong}
	\end{subfigure}
	\begin{subfigure}[b]{0.248\textwidth}
		\includegraphics[width=\textwidth]{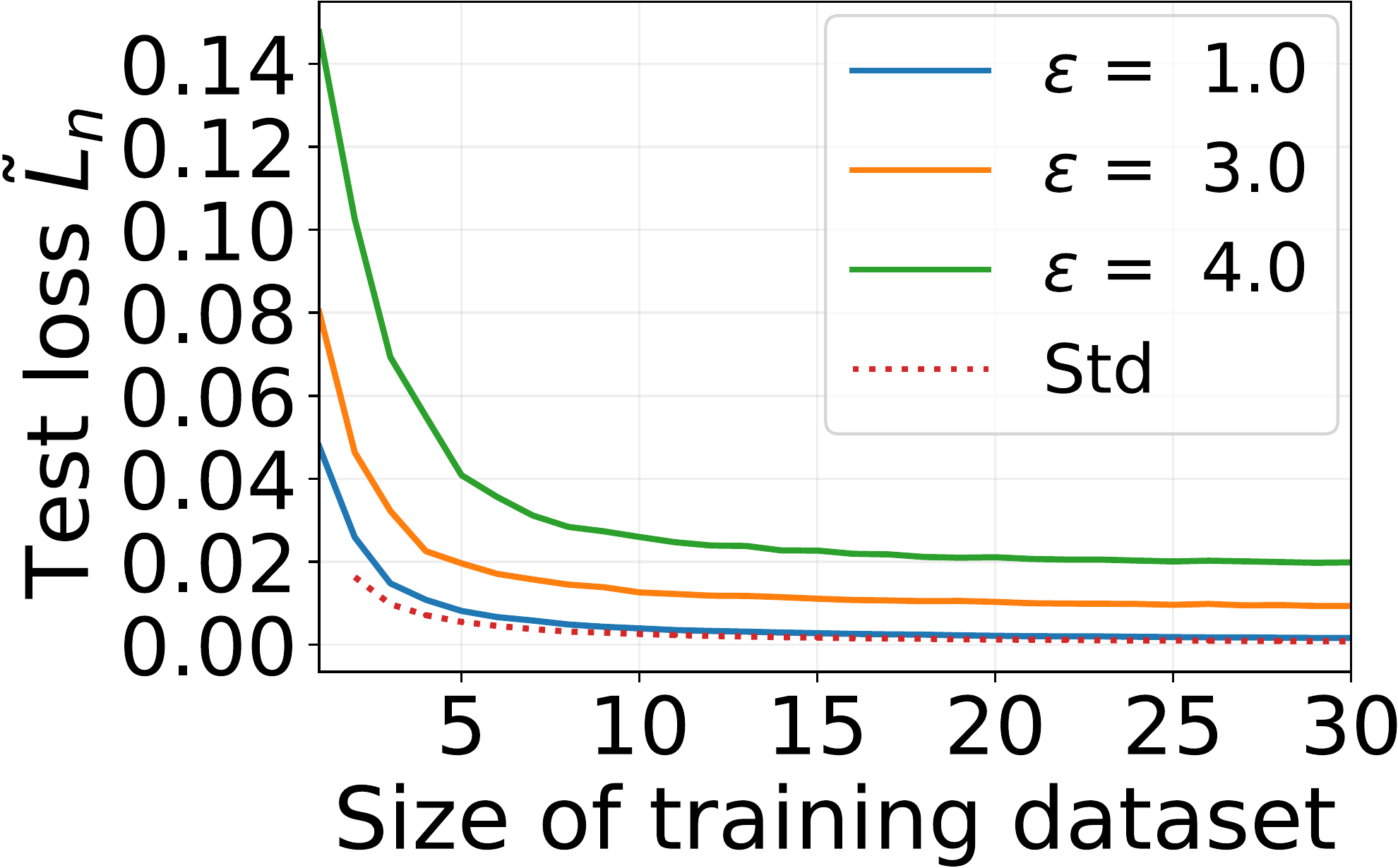}
		\caption{Poisson, small $\eps$}
		\label{fig:lr_poisson_weak}
	\end{subfigure}
	\begin{subfigure}[b]{0.252\textwidth}
		\includegraphics[width=\textwidth]{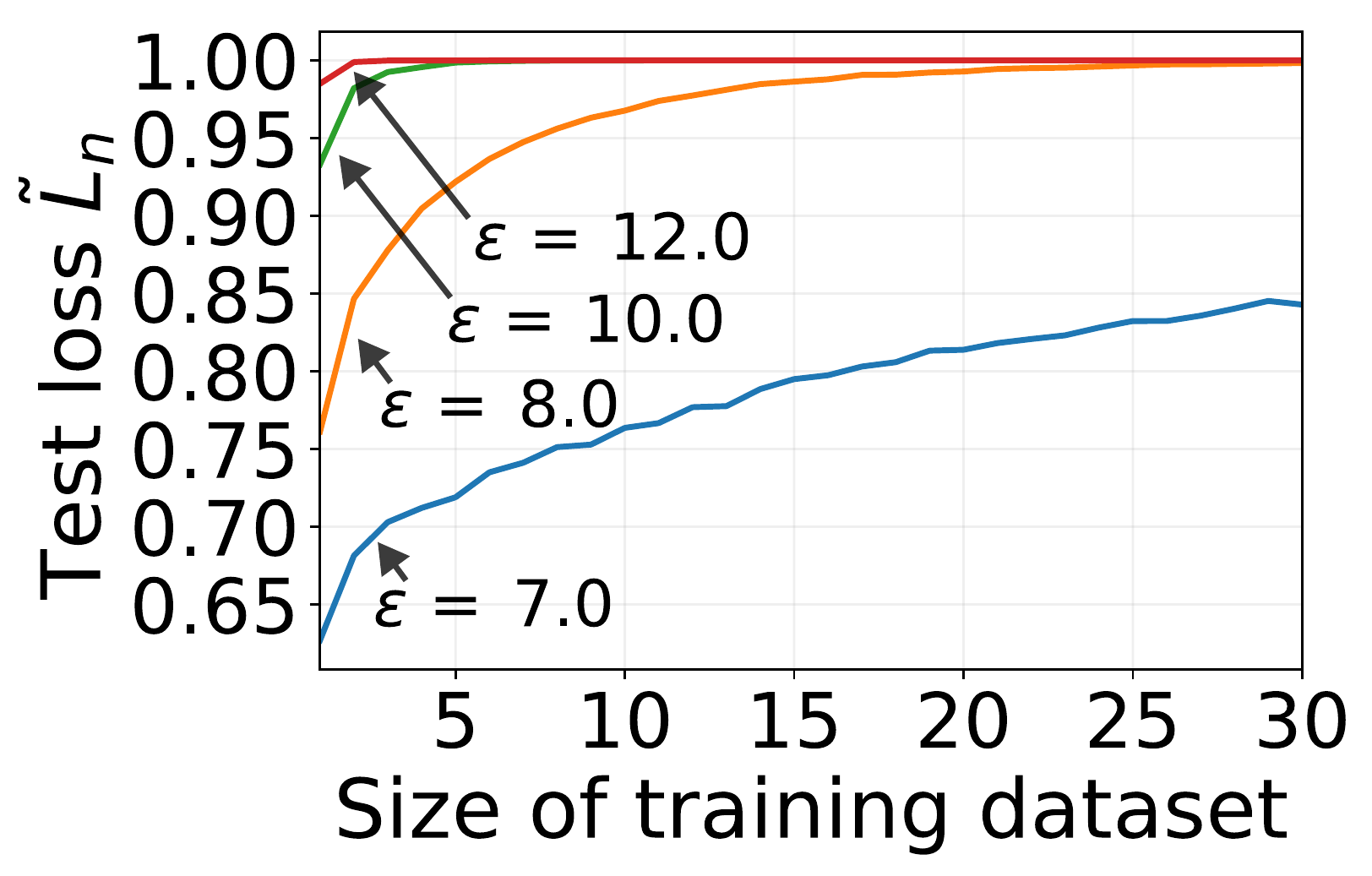}
		\caption{Poisson, large $\eps$}
		\label{fig:lr_poisson_strong}
	\end{subfigure}
	\caption{The test loss versus the size of the training dataset under 1-dim linear regression model with the squared loss. The data generation follows either a Gaussian or Poisson distribution: in \cref{fig:lr_gaussian_weak} and \cref{fig:lr_gaussian_strong}, $x \sim \cN(0,1)$; in \cref{fig:lr_poisson_weak} and \cref{fig:lr_poisson_strong}, $x \sim \text{Poisson}(5)+1$. The solid curves correspond to robust models and the dashed curve corresponds to the standard model. Here $\Tilde{L}_n = (L_n - \bE e^2)/ \bE x^2$ is the scaled test loss.} 
	\label{fig:lr}
\end{figure}

Besides classification problems, we also identify similar phenomenon in one-dimensional linear regression $y =  w^* x + e$ where $e \sim \cN(0,1)$ with squared loss $ \ell(x,y;w)  = (y- w x )^2 $. \cref{fig:lr} shows experimental results for linear regression. Given the coefficient $w$ trained on $n$ data points, the test loss is $
    L_n = \bE_{x,y}[(y-wx)^2] = \bE_{x,e}[(w^*x+e-wx)^2]
    = \bE((w^*-w)x)^2+ \bE e^2
    = (w-w^*)^2\bE x^2 + \bE e^2$.
Therefore, we report the scaled test loss $\Tilde{L}_n = (L_n - \bE e^2)/ \bE x^2 = (w-w^*)^2$ in \cref{fig:lr}. We use two different distributions for $x$: the Gaussian distribution $\cN(0,1)$ and the shifted Poisson distribution $\pois(5)+1$. We add $1$ to the outcome of $\pois(5)$ in order to guarantee a nonzero $x$. In both Gaussian and Poisson cases, we observe weak and strong regimes. When the perturbation strength $\eps$ is less than a threshold, it falls into the weak regime where the (scaled) test loss is reduced with more training data. When $\eps$ exceeds the threshold, it exhibits the strong regime where more data hurts the (scaled) test loss. However, the threshold is remarkably different for these two distributions. The threshold for $\cN(0,1)$ is between $0.6$ and $0.8$, while it resides between $4.0$ and $7.0$ for $\pois(5)+1$. This observation suggests that the Poisson data distribution appears to be more robust to adversarial perturbation. A possible explanation could be that the distribution $\pois(5)+1$ is supported on positive integers so the minimum distance between data points is $1$, while there is no such minimum distance for data points following $\cN(0,1)$ and it becomes increasingly crowded as we have more data. As a result, adversarial perturbation has a stronger influence on Gaussian data.  

\section{Conclusion}\label{sec:conclusion}

The goal of adversarial training is to produce robust models that provide protection against attacks that make perturbations to the data at test time. While protection against such attacks is undoubtedly important, we still want our robust models to perform well on unperturbed data. However, our results indicate that there are scenarios in which it is impossible for current approaches to achieve low generalization error on both datasets simultaneously. This is in direct contradiction to one of the primary tenets of machine learning, which is that more data should help us learn better. Our findings suggest that the current adversarial training framework may not be ideal and that fundamentally new ideas may be required to develop models that can reliably perform well on both perturbed and unperturbed test sets.

\section*{Acknowledgements}
We would like to thank Peter~Bartlett and Yiping~Lu for helpful comments and thank Marko~Mitrovic for his help in preparation of the paper.

\bibliography{reference}
\bibliographystyle{plainnat}

\newpage

\appendix

\section{Proof of \cref{thm:gaussian} and \cref{cor:gaussian}}\label{sec:proof_linear_loss}
Before proving \cref{thm:gaussian}, we need to establish several lemmas. First we restate the result by \citet{chen2020more} that gives the closed form solution for the robust classifier. 
\begin{prop}[Lemma~10 in \citet{chen2020more}]\label{prop:wrob}
Given $n$ training data points $\{ (x_i, y_i) \}_{i=1}^n \subset \mathbb{R}^d 
\times \{ \pm 1 \}$ and $\eps >0$, if the robust classifier is defined as 
\eqref{eq:wrob_linear_loss}, then we have $\wrob_n = W \sign (u - \eps 
\sign(u))$, where $u = \frac{1}{n} \sum_{i=1}^n y_i x_i$. 
\end{prop}

First, 
we define the error function~\citep{andrews1998special} $\erf(\cdot):  \bR \to \bR$ by 
\begin{equation}\label{eq:erf}
    \erf(x) = \frac{2}{\sqrt{\pi}} \int^x_0 e^{-t^2} \ dt \,,
\end{equation}
and it has the following property.

\begin{lemma}\label{lem:erf}
If $z \sim \cN (0,1)$, we have 
\begin{equation*}
    \mathbb{P}(z < x) = \frac{1}{2} \left[  1 + \erf \left( \frac{x}{\sqrt{2}} \right) \right] \,.
\end{equation*}
\end{lemma}
\begin{proof}[Proof of \cref{lem:erf}]
In light of the density of the standard normal distribution and by a change of 
variable, we have 
\begin{equation*}
    \mathbb{P} (z < x) ={} \frac{1}{2} + \frac{1}{\sqrt{2 \pi}} \int_0^x e^{- \frac{t^2}{2}} \ dt = \frac{1}{2} + \frac{\sqrt{2}}{\sqrt{2 \pi}} \int_0^{x/\sqrt{2}} e^{-s^2}  \ ds = \frac{1}{2} \left[  1 + \erf \left( \frac{x}{\sqrt{2}} \right) \right] \,.
\end{equation*}
\end{proof}
In addition, 
we define the function $L(\cdot , \cdot): \ \bR^2 \to \bR$ by
\begin{equation}\label{eq:L}
\begin{split}
    L(v , \epsp) ={}& \erf\left(v\right) + \erf\left(v\left( \epsp-1\right)\right) -\erf\left( v\left(\epsp+1\right)\right) \,.
\end{split}
\end{equation}
For all $j \in [d]$, we define
\begin{equation}\label{eq:v_epsp}
    v_j = \frac{\sqrt{n}\mu(j)}{\sqrt{2} \sigma(j)}\,, \quad  \epsp_j  = 
    \frac{\eps}{\mu(j)} \,,
\end{equation}
where $ \mu(j) $ and $ \sigma(j) $ are defined in the data generation process 
described at the beginning of \cref{sec:main}.

\cref{lem:rob_loss} gives the expression for the generalization error. 

\begin{lemma}\label{lem:rob_loss}
Suppose that the generalization error is defined as in 
\eqref{eq:generalization_err_linear_loss}. Then we have \[
L_n = W \sum_{j \in [d]} \mu(j) L\left( v_j , \epsp_j\right) \,,
\]
where $ v_j $ and $ \eps'_j $ are defined in \eqref{eq:v_epsp}.
\end{lemma} 

\begin{proof}[Proof of \cref{lem:rob_loss}]
By \eqref{eq:generalization_err_linear_loss}, \cref{prop:wrob} and the 
independence between test and training data, we have 
\begin{equation*}
    \begin{split}
        L_n ={}& - \bE_{\{(x_i,y_i)\}_{i=1}^n\stackrel{\textnormal{i.i.d.}}{\sim} 
		\cD_{\cN}}\left[  \bE_{(x,y)\sim \cD_{\cN}} \left[ 
		y\langle\wrob_n,x\rangle  \right] \right] = - 
		\bE_{\{(x_i,y_i)\}_{i=1}^n\stackrel{\textnormal{i.i.d.}}{\sim} 
		\cD_{\cN}}\left[  \langle \wrob_n , \mu \rangle  \right]
        \\={}& - W \cdot \sum_{j \in [d]} \mu(j)\bE_{\{(x_i,y_i)\}_{i=1}^n\stackrel{\textnormal{i.i.d.}}{\sim} 
		\cD_{\cN}}\left[  \sign \left( u(j) - \eps \sign \left( u(j)  \right)   
		\right) \right]
    \end{split}
\end{equation*} Since $y_i x_i \sim \cN(\mu, \Sigma)$, we have $u \sim \cN ( \mu , \frac{\Sigma}{n} )$, and it follows that 
\begin{equation*}
    L_n ={} - W \cdot \sum_{j \in [d]} \mu(j) \E_{u(j) \sim \cN ( \mu(j), \frac{\sigma^2(j)}{n} )} \left[ \sign\left( u(j) - \eps \sign \left( u(j) \right) \right)\right] \,. 
\end{equation*}Denote $I_j = -\E_{u(j) \sim \cN \left( \mu(j), \frac{\sigma^2(j)}{n} \right)} \left[ \sign\left( u(j) - \eps \sign \left( u(j) \right) \right)\right]$. Then we have 
\begin{equation*}
    \begin{split}
        I_j ={}& \mathbb{P} \left(  u(j) < - \eps \right) - \mathbb{P} \left(  -\eps <u(j) < 0 \right) + \mathbb{P} \left(  0 < u(j) <  \eps \right) - \mathbb{P} \left(  \eps <u(j)\right)  
        \\ ={}& 1 - 2 \mathbb{P} \left(  -\eps <u(j) < 0 \right) - 2 \mathbb{P} \left(  \eps <u(j)\right) 
        \\ ={}& 1 - 2 \mathbb{P} \left(  \frac{(-\eps - \mu(j)) \sqrt{n}}{\sigma(j)} < z < \frac{- \mu(j) \sqrt{n}}{\sigma(j)}  \right)- 2 \mathbb{P} \left(  \frac{( \eps - \mu(j)) \sqrt{n}}{\sigma(j)} < z     \right) 
        \\ = {}& 1 - 2 \left[ \mathbb{P} \left( z < \frac{( \eps + \mu(j)) \sqrt{n}}{\sigma(j)}      \right)  -  \mathbb{P} \left( z < \frac{\mu(j) \sqrt{n}}{\sigma(j)}      \right) \right] - 2 \left[  1- \mathbb{P} \left( z < \frac{( \eps - \mu(j)) \sqrt{n}}{\sigma(j)}      \right)  \right] \,,
    \end{split} 
\end{equation*}
where $ z $ is a standard normal random variable. 
 By \cref{lem:erf} we have
\begin{equation*}
    \begin{split}
        I_j ={}& \erf\left(  \frac{ \mu(j) \sqrt{n}}{\sqrt{2} \sigma(j)}  \right)  + \erf\left(  \frac{(\eps - \mu(j)) \sqrt{n}}{\sqrt{2} \sigma(j)}  \right) - \erf\left(  \frac{(\eps + \mu(j)) \sqrt{n}}{\sqrt{2} \sigma(j)}  \right) 
        \\ ={}& \erf(v_j) + \erf(v_j(\epsp_j - 1) )  - \erf(v_j(\epsp_j + 1) ) = L(v_j , \epsp_j)\,,
    \end{split}
\end{equation*} which implies that $L_n = W \sum_{j \in [d]} \mu(j) L(v_j , 
\epsp_j)$.

\end{proof}

Note that $L(v, \epsp)$ is differentiable in $v$, and by our definition each $v_j$ is smooth and monotonic in $n$. Together with \cref{lem:rob_loss} we know that $L_n$ is differentiable w.r.t. $n$. Therefore, to study the dynamic of $L_n$ in $n$, it is equivalent to studying the derivative $\frac{d L_n}{dn}$. 
We define the function $f(\cdot , \cdot):  \bR^2 \to \bR$ by
\begin{equation*}\label{eq:f_definition}
        f\left( t, \epsp\right) ={} t - (1 + \epsp) t^{(1+\epsp)^2} - 
        (1-\epsp)t^{(1-\epsp)^2} \,.
\end{equation*}

In \cref{lem:rob_loss_deri}, we compute the partial derivative of $L$.

\begin{lemma}\label{lem:rob_loss_deri} 
Let $t = e^{-v^2}$ and $ f $ be defined as in \eqref{eq:f_definition}. The 
partial derivative of $L(v,\epsp)$ w.r.t.\ $v$ is given by
\[
\frac{\partial L(v, \epsp)}{\partial v} = \frac{2}{\sqrt{\pi}} f(t , \epsp) \,.
\]
\end{lemma}

\begin{proof}[Proof of \cref{lem:rob_loss_deri}]
By \eqref{eq:erf} we have
\begin{equation*}
    \frac{d}{dx} \erf(x) = \frac{2}{\sqrt{\pi}} e^{- x^2} \,,
\end{equation*} and it follows by \eqref{eq:L} that 
\begin{equation*}
    \frac{\partial L(v , \epsp)}{\partial v} ={} \frac{2}{\sqrt{\pi}} e^{- v^2} + (\epsp - 1 )\frac{2}{\sqrt{\pi}} e^{- v^2 \cdot (\epsp -1)^2} -  (\epsp + 1 )\frac{2}{\sqrt{\pi}} e^{- v^2 \cdot (\epsp+1)^2} = \frac{2}{\sqrt{\pi}} f (t,\epsp) \,.
\end{equation*}
\end{proof}

The proof of \cref{thm:gaussian} follows from studying the derivative $\frac{d 
L_n}{d n}$. \cref{lem:rob_loss_deri}  implies that the derivative depends on 
the sign of the function $f$. We investigate the sign of $f$ in \cref{lem:f}. 

\begin{lemma}\label{lem:f}
There exist $0 < \delta_1 \leq \delta_2 < 1$ such that the following statements hold.
\begin{enumerate}[label=(\alph*),nosep]
    \item When $0 < \epsp < \delta_1$, $f(t , \epsp) < 0$ for $\forall \ t \in (0,1)$. \label{it:f_delta1}
    \item When $\delta_2 < \epsp < 1$, there exist $0 < \tau_1 < \tau_2 < 1$ depending on $\epsp$ such that  \[
    f(t , \epsp)  \begin{cases} <0 {}& \quad \forall  t \in (0,\tau_1) \,,\\
    >0 {}& \quad \forall  t \in (\tau_1,\tau_2) \,,\\
    <0 {}& \quad \forall  t \in (\tau_2,1) \,,
    \end{cases}
    \] and 
    \begin{align*}
     \lim_{\epsp \to 1^-} \tau_1(\epsp) ={}& 0 \,,\\
     \tau_2(\epsp) \geq{}& \frac{1}{3} \,.
    \end{align*}
     \label{it:f_delta2}
    \item When $1 \leq \epsp$, $f(t,\epsp)$, there exists $\tau_2 < 1$ such that \[
    f(t , \epsp)  \begin{cases} 
    >0 {}& \quad \forall  t \in (0,\tau_2) \,,\\
    <0 {}& \quad \forall  t \in (\tau_2,1) \,.
    \end{cases}
    \]\label{it:f_big_epsp}
\end{enumerate}
\end{lemma}

We compute the partial derivative of $f$ w.r.t.\ $t$  \[
f'(t, \epsp) =  \frac{\partial f (t , \epsp)}{\partial t} = 1 - \left( 1 + \epsp \right)^3 t^{(1+\epsp)^2 -1} - \left( 1-\epsp\right)^3 t^{(1-\epsp)^2 - 1} \,.
\] The proof of \cref{lem:f} uses the following \cref{lem:fprime_lim} and 
\cref{lem:fprime_trd}. To make it concise, whenever we fix $\epsp$ in the 
context, we omit $\epsp$ and write $f(t) = f(t,\epsp)$ and $f'(t) = 
f'(t,\epsp)$. 

\begin{lemma}\label{lem:fprime_lim}
	The right-sided limit of $ f' $ at $ 0 $ is given by
\begin{equation*}
    \begin{split}
        \lim_{t \to 0^+} f' (t) = \begin{cases}
        -\infty {}& \quad \textnormal{if} \ 0<\epsp<1 \,,\\ 
        1 {}& \quad \textnormal{if} \ \epsp=1 \,,\\
        +\infty {}& \quad \textnormal{if} \ 1<\epsp<2 \,.\\
        \end{cases}
    \end{split}
\end{equation*}
In addition, we have 
\begin{equation*}
    \lim_{t \to 1^{-}} f' (t) < 0 \  , \quad  \ \forall \ 0 < \epsp \,.
\end{equation*}
\end{lemma}

The proof of \cref{lem:fprime_lim} follows from direct computation. Using 
\cref{lem:fprime_lim}, we obtain \cref{lem:fprime_trd}.

\begin{lemma}\label{lem:fprime_trd}
For any fixed $0 < \epsp <1$, there exists some $t_0 = t_0(\epsp) \in (0,1)$ 
such that $f'(t)$ is strictly increasing for $t \in (0,t_0)$ and strictly 
decreasing for $t \in (t_0,1)$. For any fixed $1 \leq \epsp \leq 2$, $f' (t)$ 
is strictly decreasing for $t \in (0,1)$.
\end{lemma}

\begin{proof}[Proof of \cref{lem:fprime_trd}]
We differentiate $f'$ w.r.t. $t$ to get 
\begin{equation*}
    \frac{\partial f'(t)}{\partial t}  = - \left( 1 + \epsp\right)^3 \left[ \left(1+\epsp\right)^2 -1 \right] t^{(1+\epsp)^2 -2} - \left( 1 -\epsp\right)^3 \left[ \left( 1 -\epsp \right)^2-1\right] t^{(1-\epsp)^2 -2} \,.
\end{equation*}
First we consider the case where $0 < \epsp <1$. The function $f'$ is 
continuously differentiable on $(t,\epsp) \in (0,1)\times(0,1)$. For any fixed 
$\epsp<1$, setting $\frac{\partial f'(t)}{\partial t} = 0$ yields the unique 
solution of $t$ in $(0,1)$ as
\begin{equation}\label{eq:t0}
    t_0 = \left[ \left( \frac{1 +\epsp}{1-\epsp} \right)^3 \left( \frac{2 + 
    \epsp}{2 - \epsp}\right)\right]^{-\frac{1}{4\epsp}} \,.
\end{equation} Since $\lim_{t \to 0^+} f'(t) = -\infty$, $f'(t)$ is strictly 
increasing w.r.t.\ $t \in (0,t_0)$. Also note that \[
\begin{split}
    \lim_{t \to 1^-}\frac{\partial f'(t)}{\partial t}  ={}& \lim_{t \to 1^-}  - \left( 1 + \epsp\right)^3 \left[ \left(1+\epsp\right)^2 -1 \right] t^{(1+\epsp)^2 -2} - \left( 1 -\epsp\right)^3 \left[ \left( 1 -\epsp \right)^2-1\right] t^{(1-\epsp)^2 -2} \\
    ={}& -2 \epsp ^2 \left(5 \epsp ^2+7\right)
     < 0 \,,
\end{split}
\] which together with $\frac{\partial}{\partial t}(f'(t_0)) = 0$ indicates that $f'(t)$ is strictly decreasing for $t \in (t_0,1)$. We conclude that $t_0$ is the unique local extreme and also the global maximum of $f'(t)$ on $t \in (0,1)$.

For $1 \leq \epsp \leq 2$, we have for all $t \in (0,1)$
\begin{equation*}
    \begin{split}
         - \left( 1 + \epsp\right)^3 \left[ \left(1+\epsp\right)^2 -1 \right] t^{(1+\epsp)^2 -2} <{}& 0 \,,
        \\  - \left( 1 -\epsp\right)^3 \left[ \left( 1 -\epsp \right)^2-1\right] t^{(1-\epsp)^2 -2}  \leq{}& 0 \,.
    \end{split}
\end{equation*}
It follows that $\frac{\partial f'(t)}{ \partial t} < 0$, which implies that $f'(t)$ is strictly decreasing.

\end{proof}

A direct application of \cref{lem:fprime_trd} gives the following \cref{lem:fprime_zeros}

\begin{lemma}\label{lem:fprime_zeros}
For all  $0<\epsp<1$  sufficiently close to 1, $f' (t)$ has exactly two zeros 
on $t \in (0,1)$.
\end{lemma}

\begin{proof}[Proof of \cref{lem:fprime_zeros}]

By \cref{lem:fprime_trd}, we know that $f'(t)$ is strictly increasing on $t \in 
(0,t_0)$ and strictly decreasing on $(t_0,1)$. 
Recall that \cref{lem:fprime_lim} shows that for $ 0<\epsp<1 $, $ \lim_{t\to 
0^+} f'(t)=-\infty $ and $ \lim_{t\to 1^-} f'(t)<0 $.
Therefore it suffices to show 
$f'(t_0) > 0$ for all $\epsp$ sufficiently close to $1^-$. We define 
\begin{equation}\label{eq:defA}
    A = \left( \frac{1 +\epsp}{1-\epsp} \right)^3 \left( \frac{2  + \epsp}{2  - 
    \epsp}\right) \,.
\end{equation}
We have $ A $ tends to $+ \infty$ as $\epsp \to 1^-$. We then write
\begin{equation*}
    \begin{split}
        f'(t_0) = 1-(1+\epsp)^3 A^{ - \frac{1}{2} - \frac{\epsp}{4}} - 
        (1-\epsp)^3 A^{\frac{1}{2} - \frac{\epsp}{4}}\,.
    \end{split}
\end{equation*} 
Note that $\lim_{\epsp \to 1^-}(1+\epsp)^3 A^{ - \frac{1}{2} - \frac{\epsp}{4}} 
= 0$, and \[
\lim_{\epsp \to 1^-} (1-\epsp)^3 A^{ \frac{1}{2} - \frac{\epsp}{4}} = 
\lim_{\epsp \to 1^-} (1-\epsp)^{\frac{3}{2} + \frac{3 \epsp}{4}} \cdot \left[ 
(1 + \epsp)^3 \left( 1 + \frac{2 \epsp}{2  - \epsp} \right)\right]^{\frac{1}{2} 
- \frac{\epsp}{4}} = 0 \,.
\]Therefore we conclude that $f'(t_0) > 0$ as $\epsp \to 1^-$.

\end{proof}

We denote the two zeros in \cref{lem:fprime_zeros} by $t_1 = t_1 (\epsp)$ and 
$t_2 = t_2(\epsp)$ where $t_1 < t_2$. 

Now we are ready to prove \cref{lem:f}. 
\begin{proof}[Proof of \cref{lem:f}]
We show \ref{it:f_delta1} first. Note that for any fixed $\epsp<1$, $f(0) = 0$. 
Therefore it suffices to show that for any $\epsp$ sufficiently close to $0$, 
the derivative
$ f'(t)< 0 $.
  Since by 
\cref{lem:fprime_trd} we have  $f'(t)<\sup_{t \in (0,1)} f' (t) = f'(t_0)$ 
when $0<\epsp <1$ , it remains to show that $f'(t_0) < 0$ for all $\epsp$ 
sufficiently close to 0. 

In light of \eqref{eq:t0}, $f'(t_0) < 0$ is equivalent to \[
1 - (1 +\epsp)^3 \left[ \left( \frac{1 +\epsp}{1-\epsp} \right)^3 \left( 
\frac{2  + \epsp}{2  - \epsp}\right)\right]^{- \frac{ \epsp^2 + 2 \epsp}{4 
\epsp}} - (1-\epsp)^3 \left[ \left( \frac{1 +\epsp}{1-\epsp} \right)^3 \left( 
\frac{2  + \epsp}{2  - \epsp}\right)\right]^{-\frac{\epsp^2 - 2 
\epsp}{4 \epsp}} < 0 \,.
\]
Recall that we define
 \[
A =  \left( \frac{1 +\epsp}{1-\epsp} \right)^3 \left( \frac{2  + \epsp}{2  - 
\epsp}\right) \,.
\] 
 Rearranging the terms yields $A^{{\epsp}/{4}} < (1+\epsp)^3 A^{-1/2} + 
 (1-\epsp)^3 A^{1/2}$. Since $A>1$ and $\epsp<1$, we have $A^{\epsp/4} < 
 A^{1/2}$. Thus it now suffices to show $A^{1/2} < (1+\epsp)^3 A^{-1/2} + 
 (1-\epsp)^3 A^{1/2}$, or equivalently $A<(1+\epsp)^3/[1-(1-\epsp)^3]$. We can 
 further simplify this into \[
\frac{2  + \epsp}{2  - \epsp} < \frac{(1-\epsp)^3}{1-(1-\epsp)^3} \,.
\]Finally, note that  $\mathrm{LHS}\to 1$ and  $\mathrm{RHS}\to + \infty$ as 
$\epsp \to 0^+$. 
Therefore there must exist $\delta_1 \in (0,1)$ such that: for any $0< \epsp < 
\delta_1$, $f'(t)<0$ for all $t \in (0,1)$. Thus $f(t) <0$ for all $t\in(0,1)$.

Now we show \ref{it:f_delta2}. 
By \cref{lem:fprime_zeros}, we know that for all $\epsp$ sufficiently close to 
$1^-$, $f'$ has exactly two zeros $t_1$ and $t_2$. By \cref{lem:fprime_trd}, we 
know that $f' (t) > 0$ for $t \in (t_1, t_2)$. These imply that $f(t)$ is 
decreasing on $t \in (0,t_1)$, increasing on $t \in (t_1, t_2)$ and 
decreasing on $t \in (t_2,1)$, which gives $\argmax_{t \in [0,1]} f(t) 
\subseteq \{ 0,  t_2 \}$.
Furthermore, since $f(0)=0$ and $f'(t) <0$ for $t \in (0,t_1)$, we know 
$f(t) < 0$ in $t \in(0,t_1)$. Also note that $f(1) = -1 <0$. Therefore, 
depending on $\epsp$, the sign of $f(t)$ in $t \in(0,1)$ only has two 
possibilities: either $f(t) < 0$ for all $t\in (0,1)$ except possibly one point 
where $f(t) = 0$, or there exist $\tau_1 $ and $ \tau_2$ as described in 
\ref{it:f_delta2}. In the latter case we have $0 < t_1 < \tau_1 < t_2 < \tau_2 
< 1$.

We now show the existence of such $\tau_1$ and $\tau_2$ for all $\epsp$ sufficiently close to $1^-$. Since we have shown that $\argmax_{t \in [0,1]} f(t) \subseteq \{ 0, t_2 \}$ and $f(0)=0$, it suffices to show $f(t_2) > 0$. Since $f'(t_2) = 0$, we have $f(t_2) >0 \Leftrightarrow f(t_2) - t_2 \cdot f' (t_2) > 0 \Leftrightarrow [(1+\epsp)^3 - (1+\epsp)]t_2^{(1+\epsp)^2} > [(1-\epsp) - (1-\epsp)^3] t_2^{(1-\epsp)^2}$, which can be simplified into\[
\frac{\left( 1+\epsp\right)^3 - \left( 1 + \epsp \right)}{\left( 1 -\epsp \right) - \left( 1-\epsp\right)^3} > \frac{1}{t_2^{4 \epsp}} \,.
\] Since $\epsp <1$, it then suffices to show \[
1 + \frac{6}{\frac{2}{\epsp} + \epsp -3} \geq \frac{1}{t_2^4} \,.
\] Observe that $\mathrm{LHS}\to + \infty$ as $\epsp \to 1^-$. It remains to show that $t_2$ is bounded away from $0$ as $\epsp \to 1^-$, i.e., $\liminf_{\epsp \to 1^-} t_2 (\epsp) > 0 $. We claim that $\liminf_{\epsp \to 1^-} t_2 \geq \frac{1}{2}$. To show this, we note that 
\begin{equation*}
    \begin{split}
        \liminf_{\epsp \to 1^-} f' \left(q , \epsp \right) = \liminf_{\epsp \to 1^-} 1 - (1+\epsp)^3 \cdot {q^{(1+\epsp)^2 - 1}} - (1-\epsp)^3 \cdot {q^{(1-\epsp)^2 - 1}} = 1 - 2^3 \cdot {q^3}  \,,
    \end{split}
\end{equation*} which equals zero when $q=\frac{1}{2}$.

The claim in \ref{it:f_delta2} that $ \tau_2 (\epsp) \geq \frac{1}{3}$ follows directly from the above analysis since $t_2 < \tau_2$ and $\liminf_{\epsp \to 1^-} t_2 \geq \frac{1}{2}$. 

To show $\lim_{\epsp \to 1^-} \tau_1 (\epsp) = 0$, we claim that $\tau_1 \leq (1-\epsp)^{0.9}$ as $\epsp \to 1^-$. Then it suffices to show that $f((1-\epsp)^{0.9} , \epsp) > 0$ for all $\epsp \to 1^-$. We have 
\begin{equation*}
    \begin{split}
       \frac{1}{(1-\epsp)^{0.9}}\cdot f((1-\epsp)^{0.9} , \epsp) ={}& 1 - (1+\epsp) (1-\epsp)^{0.9 [(1+\epsp)^2 - 1]} -  (1-\epsp)^{1 + 0.9[(1-\epsp)^2-1]} \,,
    \end{split}
\end{equation*}which tends to 1 as $\epsp \to 1^-$. This implies \ref{it:f_delta2}. 

We now show \ref{it:f_big_epsp}. First note that $f(0)=0$ and $f(1)= -1$.

When $\epsp=1$, $f(t)=t - 2t^4$. In this case, we have $f(t) > 0$ for $t \in (0,2^{-1/3})$ and $f(t) <0$ for $t \in (2^{- 1/3} , 1)$.

When $1 < \epsp \leq 2 $, by \cref{lem:fprime_trd}, we have $f'(t) = 1 + (\epsp -1)^3 t^{(\epsp-1)^2-1} - (\epsp+1)^3 t^{(\epsp+1)^2 - 1}$ being strictly decreasing on $t \in (0,1)$. Therefore the function $f(t)$ is concave. Since $\lim_{t \to 0^+} f'(t) >0$, $f(0)=0$ and $f(1) = -1 < 0$, the result follows by concavity. 

When $2 <\epsp$, again since $f(0)=0$ and $f(1)=-1$, it suffices to show $f$ is strictly increasing and then strictly decreasing on $t \in (0,1)$. Note that since $\lim_{t \to 0^+} f' (t) =1 >0$ and $\lim_{t \to 1^-} f'(t) < 0  $, it then suffices to show $f'(t)$ is increasing and then decreasing on $(0,1)$. To show this, it suffices to show that if $ f''(\hat{t}) = \frac{\partial}{\partial t} f(\hat{t})<0$ for some $\hat{t} \in (0,1)$, then $f'' (t) < 0$ for all $t \in [\hat{t} , 1)$. Now, since \[
f''(\hat{t}) < 0 \Leftrightarrow \frac{(\epsp-1)^3 \left[ (\epsp-1)^2 - 1  \right]}{(\epsp+1)^3 \left[ (\epsp+1)^2 - 1  \right]} < \hat{t}^{(\epsp+1)^2 - (\epsp-1)^2} \,,
\]and $\hat{t}^{(\epsp+1)^2 - (\epsp-1)^2} < t^{(\epsp+1)^2 - (\epsp-1)^2}$ for all $t \geq \hat{t}$, we conclude that $f''(t) < 0$ for all $t \in  [ \hat{t},1)$. So we are done.

\end{proof}

Now we are in a position to prove \cref{thm:gaussian}. 

\begin{proof}[Proof of \cref{thm:gaussian}]
Let $t_j = e^{- v_j^2}$ for all $j \in [d]$. By \cref{lem:rob_loss} and \cref{lem:rob_loss_deri}, we have 
\begin{equation}\label{eq:L_and_f}
\begin{split}
    \frac{d L_n}{d n } ={}& W \sum_{j \in [d]} \mu (j) \frac{\partial L(v_j , \epsp_j)}{\partial v_j} \cdot \frac{d v_j}{ dn} = \frac{2 W}{\sqrt{\pi}} \sum_{j \in [d]} \mu(j) f( t_j , \epsp_j) \cdot \frac{\mu(j)}{2 \sqrt{2 }\sigma(j) \sqrt{n}} \,,
    \\ ={}& \frac{W}{\sqrt{2 n \pi}} \sum_{j \in [d]} \frac{\mu^2(j)}{\sigma(j)} f(t_j , \epsp_j) \,. 
\end{split}
\end{equation} By part \ref{it:f_delta1} of \cref{lem:f}, when $\eps < \delta_1 
\min_{j \in [d]} \mu(j) $, we have for all $j \in [d]$, it holds that $\epsp_j < 
\delta_1$ and thus $ f(t_j , \epsp_j)<0 $ for all $t \in (0,1)$. Combining it with 
\eqref{eq:L_and_f} yields $\frac{d L_n}{ dn } < 0$. 

When $\max_{j \in [d]} \mu (j) \leq \eps$, we have for all $j \in [d]$, it holds that $1 < \epsp_j$. It follows from part \ref{it:f_big_epsp} of \cref{lem:f} that for all $j\in [d]$, there exists $\tau_2 (\epsp_j)$ such that $f(t_j, \epsp_j)>0 \ \forall \ t_j \in (0, \tau_2 (\epsp_j)  )$. Pick $\tau_2 = \min_j \tau_2 (\epsp_j)$. Then for all $j\in [d]$, we have $f(t_j, \epsp_j)>0$ when $t_j < \tau_2$. Since $t_j = e^{- v_j^2} = \exp ( - \frac{n \mu^2(j)}{2\sigma^2 (j)}) $, when $\exp ( - \frac{n \mu^2(j)}{2\sigma^2 (j)}) < \tau_2$, or equivalently $n > 2\log \left( \frac{1}{\tau_2} \right) \max_{j \in [d]} \frac{\sigma^2(j)}{\mu^2 (j)}$, we have $\frac{d L_n}{ d n} > 0 $ . 

When $\delta_2 \cdot \max_{j \in [d]} \mu(j) < \eps < \min_{j \in [d]} \mu(j)$, we have for all $j \in [d]$, it holds that $\delta_2 < \epsp_j < 1$. Then by part \ref{it:f_delta2} of \cref{lem:f}, for all $j \in [d]$, $\exists \ \tau_1(\epsp_j)$ and $\tau_2 (\epsp_j)$ such that 
\begin{equation}\label{eq:pf_f_j}
    \begin{split}
        f(t_j , \epsp_j) \begin{cases} <0 {}& \quad \forall  t \in (0,\tau_1(\epsp_j)) \,,\\
    >0 {}& \quad \forall  t \in (\tau_1(\epsp_j),\tau_2(\epsp_j)) \,,\\
    <0 {}& \quad \forall  t \in (\tau_2(\epsp_j),1) \,,
    \end{cases}
    \end{split}
\end{equation}where $\tau_1(\epsp_j) \to 0^+$ as $\epsp_j \to 1^-$ and 
$\tau_2(\epsp_j) > \frac{1}{3}$, for all $j\in [d]$. Let $\tau_2 =\max_{j \in [d] } 
\tau_2 (\epsp_j)  >\frac{1}{3}$, $\tau_1 = \min_{j \in [d]} \tau_1 (\epsp_j)$ 
and $\hat{\tau}_1 = \max_{j \in [d]} \tau_1 (\epsp_j) $. Note that since 
$\lim_{\epsp_j \to 1^-} \tau_1 (\epsp_j) = 0$, without loss of generality we 
can assume $\hat{\tau}_1 < \frac{1}{3}$. It follows from \eqref{eq:pf_f_j} 
that for all $j\in [d]$
\begin{equation}\label{eq:pf_f_j_2}
    \begin{split}
        f(t_j , \epsp_j)  \begin{cases} <0 {}& \quad \forall  t \in (0,\tau_1) \,,\\
    >0 {}& \quad \forall  t \in \left(\hat{\tau}_1,\frac{1}{3} \right) \,,\\
    <0 {}& \quad \forall  t \in (\tau_2 ,1) \,.
    \end{cases}
    \end{split}
\end{equation}
Denote $\gamma = \frac{\mu(j)}{\sigma (j)}$ for all $j\in [d]$ since 
this ratio is fixed. Then we have $t_j = \exp \left( - \frac{\mu^2 
(j)n}{2\sigma^2(j)  } \right) = \exp (- \gamma^2 n/2)$. Therefore we can choose 
$N_4 = \log (\tau_1^{-1}) \cdot \left( \frac{2}{\gamma^2} \right)$, $N_3 = \log 
(\hat{\tau}_1^{-1}) \cdot \left( \frac{2}{\gamma^2}  \right)$, $N_2 = \log(3) 
\cdot \left( \frac{2}{\gamma^2}  \right)$ and $N_1 = \log (\tau_2^{-1}) \cdot 
\left( \frac{2}{\gamma^2} \right)$ where $N_1 < N_2 < N_3 < N_4$ and the result 
follows from \eqref{eq:L_and_f} and \eqref{eq:pf_f_j_2}. 

\end{proof}

\begin{proof}[Proof of \cref{cor:gaussian}]
From the proof of \cref{thm:gaussian}, in this simplified case we have $\tau_1 = \hat{\tau}_1$ and $\tau_2 = \tau_2 (\epsp_j)$ for all $j$. It follows that the thresholds $N_1, \ N_2, \ N_3,$ and $N_4$ in \cref{thm:gaussian} satisfy $N_1 = N_2$, and $N_3$ is no longer needed and can be replaced by $N_4$. Therefore only two thresholds are needed in \cref{cor:gaussian}. We denote the two thresholds as $N_1$ and $N_2$. 

It remains to show $\lim_{\eps \to \mu_0^-} N_2 (\eps) - N_1 (\eps) = + 
\infty$. From part \ref{it:f_delta2} of \cref{lem:f} and \eqref{eq:L_and_f}, we 
know the derivative $\frac{d L_n}{dn}$ is positive when $t := \exp (- \frac{n 
\mu_0^2}{2\sigma_0^2}) \in (\tau_1 , \tau_2)$, or equivalently $n \in \left( 
\log(\frac{1}{\tau_2}) \frac{2\sigma_0^2}{\mu_0^2}, \log(\frac{1}{\tau_1}) 
\frac{2\sigma_0^2}{\mu_0^2} \right)$. By \ref{it:f_delta2} of \cref{lem:f}, we 
know $\tau_1 \to 0^+$ as $\eps \to \mu_0^-$ while $\tau_2$ is bounded away from 
$0$. This shows $\lim_{\eps \to \mu_0^-} \log(\frac{1}{\tau_1}) - 
\log(\frac{1}{\tau_2}) = + \infty $ and  completes the proof. 
\end{proof}

\section{Proof of \cref{lem:frob_well_defined}}\label{sec:proof_frob_equivalent_definition}

Let $f^* \in S_2$, i.e., $f^*$ is a minimizer of $\sum_{i=1}^n \max_{||\tilde{x}_i-x_i||_\infty < \eps} H\left(-y_i (\tilde{t}_{i} - f(\tilde{s}_{i})) \right)$ with the smallest $\ell_1$ norm. To show $S_2$ is nonempty and such $f^*$ does exist, we specify the form of $f^*$. We claim that $f^*$ can take the following form
\begin{equation}\label{eq:form_of_frob}
    \begin{split}
        f^*(s) & = \sum_{j=1}^N \alpha_j \ind[s \in I_j] , 
    \end{split}
\end{equation}where $I_j = (j-\eps, j + \eps)$, $j\in [N]$. Indeed, by definition of $H$, we know that the value of $f^*$ outside those intervals $I_j$'s won't change the value of $H\left(-y_i (\tilde{t}_{i} - f(\tilde{s}_{i})) \right)$. Therefore in order to attain the smallest possible $\ell_1$ norm, we must have $f^*(s) = 0$ for all $s \notin \cup_j I_j$. 

Note that by letting $\eps <1/2$, any two intervals have no overlap. To see why $f^*$ is a constant function over each interval $I_j$, we consider three possible cases of the dataset $\{ (x_i,y_i), i\in [n] \}$. For the first case, suppose that those data points with $s = j$ contain only positive points. Then in order to correctly classify these points with $\eps$ perturbation, we must have $f^*(s) \leq \mu -\eps$ for all $s \in I_j$. In order to minimize $||f^*||_1$, we would take $\alpha_j = \min\{0,\ \mu-\eps \}$. Similarly, if those points purely consist of negative points, then $\alpha_j = \max\{0,\ -\mu+\eps \}$. For the second case, suppose that those data points with $s = j$ contain both positive and negative points. Suppose the number of positive points exceeds the number of negative points. Then to correctly classify the positive points, we have $f^*(s) \leq \mu -\eps$ for all $s \in I_j$. To correctly classify the negative points, we have $f^*(s) \geq -\mu +\eps$ for all $s \in I_j$. If $-\mu +\eps \leq 0 \leq \mu - \eps$, then $\alpha_j = 0$. Otherwise, if $-\mu +\eps > \mu - \eps$, then $f^*$ can never simultaneously classify both classes correctly. It will choose to correctly classify the class with more points, which is the positive class. Then $\alpha_j = \mu -\eps$. On the other hand, if negative class has more points, then $\alpha_j = -\mu + \eps$. If the two class have equal number of points at $s=j$, then $\alpha_j$ can be either $-\mu + \eps$ or $\mu-\eps$. For the third case, assume no point in the training set has $s=j$. Then $\alpha_j = 0$. 

We have now specified the form that $f^* \in S_2$ can take, which also indicates that $S_2$ is nonempty. We now show for all sufficiently small $\lambda$, $S(\lambda) = S_2$. 

First we show $S(\lambda) \subseteq S_2$. Let $f \in S(\lambda)$. We want to show $f \in S$ and $||f||_1 \leq || \hat{f}||_1$ for all $\hat{f} \in S$. Suppose on the contrary that $f \notin S$. Then by definition of $H$, there exists $f^* \in S$ s.t. 
\begin{equation*}
    \sum_{i=1}^n \max_{||\tilde{x}_i-x_i||_\infty < \eps} H\left(-y_i (\tilde{t}_{i} - f^* (\tilde{s}_{i})) \right) \leq \sum_{i=1}^n \max_{||\tilde{x}_i-x_i||_\infty < \eps} H\left(-y_i (\tilde{t}_{i} - f(\tilde{s}_{i})) \right) - 1/2
\end{equation*}and since $S_2$ is nonempty we can further assume $f^*$ satisfies
\begin{equation*}
    ||f^*||_1 \in \argmin_{\hat{f}\in S} ||\hat{f}||_1 .
\end{equation*}Since $f \in S(\lambda)$, we then have $\lambda ||f||_1 \leq \lambda||f^*||_1 - 1/2$, which implies $||f^*||_1 \geq 1/2\lambda$. From above analysis we know $f^*$ must take the form of \cref{eq:form_of_frob} where $\alpha_j \leq |\mu-\eps|$, and $I_j$ has length equal to $2 \eps$. This implies $||f^*||_1 \leq 2N\eps|\mu-\eps|$.  Therefore, if we pick $\lambda < \frac{1}{4 N \eps |\mu-\eps|} $, then such $f^*$ cannot exist. Therefore, for all sufficiently small $\lambda$, we have $f \in S$. 

Now we show $||f||_1 \leq || \hat{f}||_1$ for all $\hat{f} \in S$. Suppose on the contrary that there exists $f^* \in S$ such that $||f^*||_1 < ||f||_1$. However, since we have already shown 
\begin{equation*}
    \sum_{i=1}^n \max_{||\tilde{x}_i-x_i||_\infty < \eps} H\left(-y_i (\tilde{t}_{i} - f^* (\tilde{s}_{i})) \right) = \sum_{i=1}^n \max_{||\tilde{x}_i-x_i||_\infty < \eps} H\left(-y_i (\tilde{t}_{i} - f(\tilde{s}_{i})) \right) ,
\end{equation*} this would contradict the fact that $f \in S(\lambda)$. Therefore we have $S(\lambda) \subseteq S_2 $. 

To see $S_2 \subseteq S(\lambda)$ for all sufficiently small $\lambda$, we again pick $\lambda < \frac{1}{4 N \eps |\mu-\eps|}$. Note that since $||f^*||_1 \leq 2N\eps|\mu-\eps|$ for all $f^* \in S_2$, we have $\lambda||f^*||_1 < \frac{1}{2}$. Now suppose on the contrary that there exists $f \notin S_2$ such that
\begin{equation*}
    \sum_{i=1}^n \max_{||\tilde{x}_i-x_i||_\infty < \eps} H\left(-y_i (\tilde{t}_{i} - f(\tilde{s}_{i})) \right) + \lambda ||f||_1 < \sum_{i=1}^n \max_{||\tilde{x}_i-x_i||_\infty < \eps} H\left(-y_i (\tilde{t}_{i} - f^*(\tilde{s}_{i})) \right) + \lambda ||f^*||_1.
\end{equation*}Since $f^* \in S_2$, we must have $\lambda ||f||_1 < \lambda ||f^*||_1 < \frac{1}{2}$. Now, if $\sum_{i=1}^n \max_{||\tilde{x}_i-x_i||_\infty < \eps} H\left(-y_i (\tilde{t}_{i} - f(\tilde{s}_{i})) \right) \leq \sum_{i=1}^n \max_{||\tilde{x}_i-x_i||_\infty < \eps} H\left(-y_i (\tilde{t}_{i} - f^*(\tilde{s}_{i})) \right)$, this would contradict the fact that $f^* $ is in $\argmin_S ||f||_1$. Therefore we must have $\sum_{i=1}^n \max_{||\tilde{x}_i-x_i||_\infty < \eps} H\left(-y_i (\tilde{t}_{i} - f(\tilde{s}_{i})) \right) > \sum_{i=1}^n \max_{||\tilde{x}_i-x_i||_\infty < \eps} H\left(-y_i (\tilde{t}_{i} - f^*(\tilde{s}_{i})) \right)$. However, by definition of $H$, this implies 
\begin{equation*}
    \begin{split}
    \sum_{i=1}^n \max_{||\tilde{x}_i-x_i||_\infty < \eps} H\left(-y_i (\tilde{t}_{i} - f(\tilde{s}_{i})) \right) + \lambda ||f||_1 &\geq \sum_{i=1}^n \max_{||\tilde{x}_i-x_i||_\infty < \eps} H\left(-y_i (\tilde{t}_{i} - f^*(\tilde{s}_{i})) \right) + \frac{1}{2} + \lambda ||f||_1
    \\ & \geq \sum_{i=1}^n \max_{||\tilde{x}_i-x_i||_\infty < \eps} H\left(-y_i (\tilde{t}_{i} - f^*(\tilde{s}_{i})) \right) + \lambda ||f^*||_1,
    \end{split}
\end{equation*}which is a contradiction. Therefore $S_2 \subseteq S(\lambda)$. Altogether we have $S(\lambda) = S_2$.

\section{Proof of \cref{thm:brick-wall-test-loss}}\label{sec:proof_brick-wall-test-loss}

The proof follows from the \cref{lem:frob_well_defined} and its proof. 
By \cref{lem:frob_well_defined}, we have $S(\lambda) = S_2$ and we can consider the equivalent definition that $\frob_n \in S_2$. From the proof of \cref{lem:frob_well_defined}, we know $\frob_n$ must take the form of \eqref{eq:form_of_frob}. Since $|\alpha_j| \leq |\mu-\eps|$, when $\eps < 2 \mu$, we have $|\alpha_j| < \mu$ and thus $|\frob_n(s)| < \mu$ for all $s \in \bR$. For such $\frob_n$, we have $H \left(  -y\left( t - \frob_n(s) \right) \right) = 0$ for all $(x,y) = (s,t,y)$ in the support of $\cD_{2N}$. This implies $L_n = 0$ for all $n$. 

Assume $ 2 \mu < \eps < 1/2$. Then $|\alpha_j|$ can take the value of either $0$ or $|\mu-\eps| > |\mu|$. When $\alpha_j = 0$, $\frob_n$ can classify both the positive and negative points at location $s=j$ correctly. When $|\alpha_j|> \mu$, then $\frob_n$ can only classify one of the two classes correctly. Note that $\alpha_j = 0$ if and only if there is no point with $s=j$ in the training set. Let the random variable $Z \in {0}\cup[N]$ denote the cardinality of the set $\{j\in [N]: \ s_i\neq j \ \textnormal{for all} \ i \in [n] \} $, which is a function of the training set $\{(x_i,y_i)\}_{i=1}^n$. Then the generalization error can be written as 
\begin{equation*}
    L_n = \bE_{\{(x_i,y_i)\}_{i=1}^n\stackrel{\textnormal{i.i.d.}}{\sim} 
		\cD_{2N}} \frac{N-Z}{N} = 1 - \frac{\bE_{\{(x_i,y_i)\}_{i=1}^n} Z}{N} . 
\end{equation*}Note that $\bE_{\{(x_i,y_i)\}_{i=1}^n} Z$ decreases as $n$ increases. Therefore $L_{n} < L_{n+1}$ for all $n$.

\section{Further Details on Gaussian Mixture with 0-1 Loss}

\subsection{Robust Classifier under 0-1 Loss}\label{sec:minimizer_is_interval}

If the training dataset is $ \{(x_i,y_i)\}_{i=1}^n $, we 
define the neuralized dataset $ \{(x'_i,y_i)\}_{i=1}^n $ that satisfies $ 
x'_i=x_i-y_i\eps $ for all $i\in [n] $. In other words, for a positive 
sample $ (x_i, y_i=1) $, we obtain its neutralized sample by shifting $ x_i $ 
to the negative direction by $ \eps $, i.e., $ x'_i=x_i-\eps $; for a negative 
sample $ (x_i,y_i=-1) $, its neutralized sample is obtained by shifting $ x_i $ 
to the positive direction by $ \eps $, i.e., $ x'_i=x_i+\eps $. We see that the 
dataset remains unchanged after neutralization if $ \eps=0 $. With this definition, the robust classifier can be expressed as the following.

\begin{prop}\label{prop:rob-0-1}
	Given the training dataset $ \{(x_i,y_i)\}_{i=1}^n $ and the neuralized 
	dataset $ \{(x'_i,y_i)\}_{i=1}^n $, 
	the robust classifier is given by \begin{equation}\label{eq:wrob-0-1}
	\wrob_n \in \argmin_{w\in \bR} \sum_{i=1}^{n} y_i\ind[x'_i< w]\,.
	\end{equation}
\end{prop}

Now one can see the tiebreaking issue in light of \cref{prop:rob-0-1}. To see 
this, let $ s $ be the permutation of $ [n] $ such that $ x'_{s(1)}\le 
x'_{s(2)} \le \dots \le x'_{s(n)} $. 
The $ n $ points divide the real line into $ n+1 $ intervals: $ (-\infty, 
x'_{s(1)}] $, $ (x'_{s(i)}, x'_{s(i+1)}] $ for $ 1\le i\le n-1 $, and $ 
(x'_{s(n)},\infty) $. 
Let $ w^* $ be a minimizer of 
\eqref{eq:wrob-0-1}. If $ w^* $ lies in any of the above $ n+1 $ intervals, then any 
other point in the same interval is also a minimizer, since at these two points the objective function has the same value. Therefore, a tiebreaking  procedure is required here. 

For the agnostic tiebreak, if $ w^*\in (x'_{s(i)}, x'_{s(i+1)}] $, it chooses $ \wrob_n $ uniformly at random from the interval. If 
$ w^* > x'_{s(n)} $, it chooses $ \wrob_n $ arbitrarily close to $ x'_{s(n)} $ 
from above. If $ w^* \le x'_{s(1)} $, it chooses $ \wrob_n = x'_{s(1)} $.

\begin{proof}[Proof of \cref{prop:rob-0-1}:]
By \eqref{eq:wrob_general}, it suffices to show that under the 0-1 loss  
\begin{equation}\label{eq:pf_prop_rob_1}
    \begin{split}
    \sum_{i=1}^n \max_{\tilde{x}_i\in B^\infty_{x_i}(\eps)} \ind[y_i(\tilde{x}_i-w)<0] = \sum_{i=1}^{n} y_i\ind[x'_i< w] \,.
    \end{split}
\end{equation}Conditioning on whether there exists $\tilde{x}_i\in B^\infty_{x_i}(\eps)$ such that $\ind[y_i(\tilde{x}_i-w)<0] = 1$ or not, one can deduce that
\begin{equation*}
    \begin{split}
        \argmax_{\tilde{x}_i\in B^\infty_{x_i}(\eps)} 
        \ind[y_i(\tilde{x}_i-w)<0] \supseteq \argmin_{\tilde{x}_i\in 
        B^\infty_{x_i}(\eps)} y_i(\tilde{x}_i-w) = \{ x_i' \} \,,
    \end{split}
\end{equation*}and it follows that 
\begin{equation*}
    \begin{split}
        \sum_{i=1}^n \max_{\tilde{x}_i\in B^\infty_{x_i}(\eps)} \ind[y_i(\tilde{x}_i-w)<0]  = \sum_{i=1}^n \ind[y_i(x_i'-w)<0] = \sum_{i=1}^{n} y_i\ind[x'_i< w] \,.
    \end{split}
\end{equation*}
\end{proof}

\subsection{Test Loss and Optimal Tiebreak}\label{sec:0-1-test loss}

To find the optimal tiebreaking in hingsight, we need to minimize the test loss over the model parameter $w$, which is given by \cref{prop:0-1-test-loss}. 

\begin{prop}\label{prop:0-1-test-loss}
	The test loss of classifier $ w $ is given by 
	\begin{equation}\label{eq:test-loss-0-1}
	\bE_{(x,y)\sim \dgau}[\ind[y(x-w)<0]] = \frac{1}{2} + \frac{1}{2}\left( 
	\Phi\left(\frac{w-\mu}{\sigma}\right) - \Phi\left( \frac{w+\mu}{\sigma} 
	\right) \right) \,,
	\end{equation}
	where $\Phi$ is the CDF of the standard normal distribution. 
	Furthermore, the minimizer of \eqref{eq:test-loss-0-1} is $ w=0 $.
\end{prop}

\cref{prop:0-1-test-loss} indicates that the optimal tiebreak in hindsight chooses the point closest to $ 0 $ 
(i.e., the point with the minimum absolute value) from (the closure of) the interval where $ w^* 
$ lies. This is because $ w=0 $ minimizes the test loss in 
\eqref{eq:test-loss-0-1}, and one can see that \eqref{eq:test-loss-0-1} increases as $|w|$ increases. Indeed, the derivative of \eqref{eq:test-loss-0-1} is given by $\frac{1}{2 \sigma \sqrt{2 \pi}} \left( \exp({- \frac{(w-\mu)^2}{2\sigma}}) - \exp({- \frac{(w+\mu)^2}{2\sigma}}) \right)$, which is negative for $w<0$ and positive for $w>0$.

\begin{proof}[Proof of \cref{prop:0-1-test-loss}:]
Conditioning on $y = \pm 1$, we have 
\begin{equation*}
    \begin{split}
        & \bE_{(x,y)\sim \dgau}[\ind[y(x-w)<0]]\\
        = {}& \mathbb{P} (y=1) \cdot \bE_{x|y=1} \left[ \ind[y(x-w)<0]  \right] +  \mathbb{P} (y=-1) \cdot \bE_{x|y=-1} \left[ \ind[y(x-w)<0] \right]
        \\ ={}& \frac{1}{2} \cdot \bE_{x \sim \cN(\mu,\sigma)} [ \ind[x-w < 0] ] + \frac{1}{2} \cdot \bE_{x \sim \cN(-\mu,\sigma)} [ \ind[x-w > 0] ] 
        \\ = {}& \frac{1}{2} \cdot \mathbb{P}_{z \in \cN(0,1)} \left( z < \frac{w-\mu}{\sigma} \right)  +  \frac{1}{2} \cdot \mathbb{P}_{z \in \cN(0,1)} \left( z > \frac{w+\mu}{\sigma} \right)
        \\ ={}& \frac{1}{2} \cdot \Phi \left(  \frac{w - \mu}{\sigma}  \right) + \frac{1}{2} \cdot \left[ 1 - \Phi \left(  \frac{w + \mu}{\sigma}  \right)  \right] \,.
    \end{split}
\end{equation*}Since the derivative is $\frac{1}{2 \sigma \sqrt{2 \pi}} \left( \exp({- \frac{(w-\mu)^2}{2\sigma}}) - \exp({- \frac{(w+\mu)^2}{2\sigma}}) \right)$, we see that $w^* = 0$ minimizes the above quantity. 
\end{proof}

\section{Additional Details about the SVM Experiment}\label{sec:svm-detail}

We study the 2-dimensional support vector machine where the data is generated as $y \sim \unif(\{\pm 1\})$ and $X \sim \cN(y \mu , I)$ where the signal level $\mu$ is set as $\mu = (1,1)^T$. We consider the common setting of hinge loss with $\ell_2$ penalty, under which the robust classifier is defined as 
\begin{equation}\label{eq:svm-training-loss}
    \wrob_n \in \argmin_w \left[ \frac{1}{n} \sum_{i=1}^n  \max_{\| x_i' - x_i\|_\infty \leq \eps} \max \left\{ 0,\ 1- y_i \left( \langle w,x_i' \rangle -b  \right) \right\} \right] + \frac{1}{2} \lambda \| w \|_2^2. 
\end{equation} The standard test loss (the $y$-axis in \cref{fig:weak_svm_hinge_loss} and \cref{fig:strong_svm_hinge_loss}) of the robust classifier is given by 
\begin{equation*}
    \bE_{(x,y)} \max \left\{ 0,\ 1- y \left( \langle w,x \rangle -b  \right) \right\},
\end{equation*} where the penalty term is not included. 
The robust classifier $\wrob_n$ is solved for by optimizing \eqref{eq:svm-training-loss} which is convex in $w$ using gradient descent.

\end{document}